%% file: main.tex
\PassOptionsToPackage{table,dvipsnames}{xcolor}
\documentclass[10pt]{article} 
\usepackage[accepted]{tmlr}

\input{math_commands.tex}

\usepackage{hyperref}
\usepackage{url}

\usepackage[linesnumbered,lined,boxed,commentsnumbered]{algorithm2e}
\usepackage{algpseudocode}
\usepackage{microtype}
\usepackage{graphicx}
\usepackage{subcaption}
\usepackage{booktabs} 
\usepackage{arydshln}
\usepackage{hyperref}
\usepackage{cuted}
\usepackage{wrapfig}
\usepackage[nobiblatex]{xurl}

\usepackage{amsmath}
\usepackage{amssymb}
\usepackage{mathtools}
\usepackage{amsthm}
\usepackage[capitalize,noabbrev]{cleveref}
\usepackage{anyfontsize}
\usepackage{pifont} 
\usepackage{subfiles}
\usepackage{adjustbox}  
\usepackage{graphicx}
\usepackage{xspace}
\usepackage{xcolor}
\usepackage{booktabs} 
\usepackage{multirow} 
\usepackage{colortbl}
\usepackage{hyperref}
\usepackage{url}
\usepackage{caption}
\usepackage{subcaption}
\usepackage{wrapfig}
\usepackage{enumitem}
\usepackage[textsize=tiny]{todonotes}
\usepackage{placeins}

\newcommand{\method}{\textsc{CausalEmbed}\xspace}
\newcommand{\vidore}{ViDoRe}
\definecolor{CadetBlue}{HTML}{6D4E7E}
\definecolor{myred}{HTML}{DB432C}
\definecolor{myblue}{HTML}{427AB2}
\definecolor{mygreen}{HTML}{438870}
\definecolor{mypurple}{HTML}{B7617D}
\definecolor{myyellow}{HTML}{E4B112}
\definecolor{mygray}{HTML}{818181}
\definecolor{semgreen}{HTML}{56BA77}
\definecolor{mydarkred}{HTML}{982C2C}
\definecolor{tablepink}{RGB}{253, 231, 231}
\definecolor{tableyellow}{RGB}{254, 243, 215}
\definecolor{tableblue}{RGB}{222, 235, 246}
\newcommand{\red}[1]{$_{\color{RedOrange}\downarrow #1}$}

\definecolor{bittersweet}{rgb}{1.0, 0.44, 0.37}
\definecolor{mygreen}{rgb}{0.29, 0.7, 0.48}
\usepackage{pifont}
\definecolor{demphcolor}{RGB}{144,144,144}

\definecolor{mygray}{gray}{0.4}
\definecolor{autopurple}{HTML}{7030A0}
\definecolor{dyna_yellow}{HTML}{BF9000}
\definecolor{adaptive_blue}{HTML}{0070C0}
\definecolor{darksalmon}{rgb}{0.91, 0.59, 0.48}
\definecolor{emerald}{rgb}{0.31, 0.78, 0.47}
\definecolor{green(pigment)}{rgb}{0.0, 0.65, 0.31}
\definecolor{amaranth}{rgb}{0.9, 0.17, 0.31}
\definecolor{iris}{rgb}{0.35, 0.31, 0.81}
\definecolor{uu}{rgb}{0.95, 0.51, 0.51}
\definecolor{spirodiscoball}{rgb}{0.06, 0.75, 0.99}

\theoremstyle{plain}
\newtheorem{theorem}{Theorem}[section]

\theoremstyle{definition}

\theoremstyle{remark}

\title{\method: Auto-Regressive Multi-Vector Generation in Latent Space for Visual Document Embedding}

\author{\name Jiahao Huo\thanks{Equal contribution.} \email jiahaohuotj@gmail.com \\
      \addr The Hong Kong University of Science and Technology (Guangzhou)\\
      Alibaba Cloud Computing\\
      University of Illinois Chicago
      \AND
      \name Yu Huang\footnotemark[1] \\
      \addr The Hong Kong University of Science and Technology (Guangzhou)\\
      Alibaba Cloud Computing
      \AND
      \name Yibo Yan\footnotemark[1] \\
      \addr The Hong Kong University of Science and Technology (Guangzhou)\\
      Alibaba Cloud Computing\\
      The Hong Kong University of Science and Technology
      \AND
      \name Ye Pan \\
      \addr The Hong Kong University of Science and Technology (Guangzhou)
      \AND
      \name Kening Zheng \\
      \addr University of Illinois Chicago
      \AND
      \name Wei-Chieh Huang \\
      \addr University of Illinois Chicago
      \AND
      \name Yi Cao \\
      \addr Alibaba Cloud Computing
      \AND
      \name Mingdong Ou\thanks{Project Leader.} \\
      \addr Alibaba Cloud Computing
      \AND
      \name Philip S. Yu \\
      \addr University of Illinois Chicago
      \AND
      \name Xuming Hu\thanks{Corresponding author.} \email xuminghu@hkust-gz.edu.cn \\
      \addr The Hong Kong University of Science and Technology (Guangzhou)\\
      The Hong Kong University of Science and Technology}
\def\month{09}
\def\year{2026}
\def\openreview{{\normalfont\url{https://openreview.net/forum?id=mZXoOgBxX2}}}

\begin{document}

\maketitle

\begin{abstract}
    Although Multimodal Large Language Models have shown remarkable potential in Visual Document Retrieval (VDR) through generating high-quality multi-vector embeddings, the substantial storage overhead caused by representing a page with thousands of visual tokens limits their practicality in real-world applications. To address this challenge, we propose an auto-regressive generation approach, \method, for constructing multi-vector embeddings. By combining contrastive alignment with progressive refinement and diversity regularization, \method learns compact, sequential representations. Using 32 document vectors, our method reduces the vector count by approximately 32$\times$ for ColPali and 155$\times$ for ColQwen while remaining competitive across two backbones and three versions of the \vidore benchmark. Comparisons with trainable compression baselines further support its effectiveness under stringent vector budgets. We also observe an empirical budget--quality scaling behavior across both evaluated backbones, with diminishing returns at larger inference budgets. System profiling shows that the compact representations reduce index storage and late-interaction cost, while query-side auto-regressive encoding remains an important latency trade-off. Our code is available at \href{https://github.com/Z1zs/Causal-Embed}{this URL}.
\end{abstract}
\input{content/introduction}
\input{content/related_work}
\input{content/method}
\input{content/experiment}
\input{content/conclusion}


\clearpage
\bibliography{reference}
\bibliographystyle{tmlr}

\clearpage
\appendix
\input{content/appendix}
\end{document}

%% file: math_commands.tex
\usepackage{amsmath,amsfonts,bm}

\def\eqref#1{equation~\ref{#1}}

\def\1{\bm{1}}

\DeclareMathAlphabet{\mathsfit}{\encodingdefault}{\sfdefault}{m}{sl}
\SetMathAlphabet{\mathsfit}{bold}{\encodingdefault}{\sfdefault}{bx}{n}



%% file: content/introduction.tex
\section{Introduction}
Visual Document Retrieval (VDR)~\citep{barboule2025surveyvdr}, which focuses on retrieving relevant pages from extensive corpora of documents, is a cornerstone of modern information systems, ranging from enterprise search to domain-specific Retrieval-Augmented Generation (RAG) \citep{gao2025scaling,zheng2025retrieval,zhang2025roles}. Unlike traditional text-based retrieval systems that necessitate Optical Character Recognition (OCR) for content extraction~\citep{smith2007tesseract}, VDR treats document pages as visual entities~\citep{faysse2024colpali}. By formulating retrieval as a multimodal problem, VDR preserves critical structural and layout information that is often lost in text-only pipelines. Catalyzed by the rapid evolution of Multimodal Large Language Models (MLLMs), recent research~\citep{meng2025vlm2vec2, gunther2025jinav4} has pivoted toward leveraging these backbones to generate unified multimodal embeddings, building on their superior cross-modal reasoning and alignment capabilities. \par
\begin{wrapfigure}{r}{0.5\textwidth}
\centering
\includegraphics[width=0.48\textwidth]{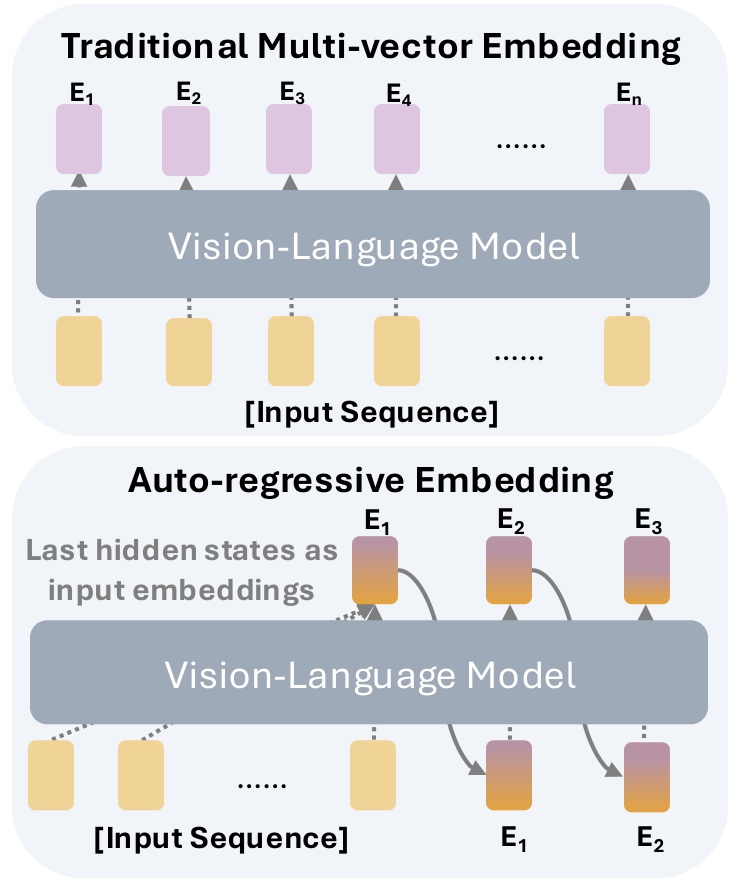}
\caption{Comparison of traditional multi-vector embeddings (\textit{e.g.,} ColPali/ColQwen \citep{faysse2024colpali}) with our auto-regressive paradigm for multi-vector generation in the VDR domain.}
\label{fig:main}
\end{wrapfigure}
Conventional single-vector VDR~\citep{jiang2024vlm2vec} encodes an entire document page and a corresponding query into a solitary vector, typically employing cosine similarity for retrieval. This paradigm has been further refined by contemporary models such as Eager Embed~\citep{EagerEmbed,meng2025vlm2vec2}, which utilizes the pre-trained backbones like Qwen3-VL~\citep{bai2025qwen3vl}. Nevertheless, these approaches are inherently constrained in their ability to represent content-dense pages, as a single embedding often fails to encapsulate the multifaceted visual complexity of a full visual page. To circumvent this, multi-vector VDR methods leverage patch-level representations. Pioneered by ColPali~\citep{faysse2024colpali}, this paradigm aligns patch-level visual tokens with textual query embeddings, treating each of the hundreds or thousands of visual tokens as a distinct representation of a document patch. This approach has demonstrated promising performance in both VDR~\citep{faysse2024colpali,mace2025vidore2} and general multimodal retrieval tasks~\citep{jiang2024vlm2vec,meng2025vlm2vec2}. The efficacy of multi-vector retrieval has been further validated by models such as EvoQwen2.5-VL-Retriever~\citep{EvoQwen}, tomoro-colqwen3~\citep{huang2025tomoro}, and ColNomic~\citep{nomicembedmultimodal2025}, which integrates hard negative mining and large-scale data curation, as well as by explorations into diverse backbones like Llama 3.2~\citep{grattafiori2024llama3} and Gemma 3~\citep{team2025gemma3} through Llama-Nemoretriever-Colembed~\citep{xu2025llama} and ColNetraEmbed~\citep{kolavi2025colnetraembed}.\par
Despite their empirical success, multi-vector representations impose a formidable storage burden, often requiring hundreds or even thousands of vectors per page~\citep{yan2025docpruner}, which complicates deployment at corpus scale. Consequently, recent research has focused on compressing these multi-vector representations. For instance, MetaEmbed~\citep{xiao2025metaembed} employs matryoshka representation learning~\citep{kusupati2022matryoshka} to facilitate multiple predefined token budgets during inference. Other strategies, such as Light-ColPali~\citep{ma2025lightcolpali} and DocPruner~\citep{yan2025docpruner}, utilize clustering, merging, or learned selection for token pruning. These methods provide effective compression, although aggressive pruning can incur retrieval loss and remains dependent on the representations produced by the underlying encoder.\par
This raises a fundamental question: can the generative prowess of MLLMs be harnessed to produce multi-vector embeddings while simultaneously mitigating the substantial redundancy inherent in patch-level representations? We posit that such an objective is attainable. A compelling direction involves auto-regressive representation generation, a methodology that has demonstrated significant effectiveness in natural language generation~\citep{brown2020gpt}, visual entity synthesis~\citep{tian2024var}, and even explicit latent reasoning~\citep{hao2024coconut}.\par
Nevertheless, the potential of the auto-regressive paradigm for embedding tasks remains largely underexplored, notwithstanding preliminary efforts~\citep{cui2025think,lan2025ume,liu2025reasoning,tsai2025giecse} that incorporate textual reasoning trajectories prior to multimodal encoding or iteratively refine single-vector text embeddings via latent reasoning.\par
\textit{In this work, we explore generating multi-vector document embeddings in a sequential, auto-regressive manner}, a framework we designate as \method. Specifically, we fine-tune a pre-trained MLLM to synthesize latent representations token by token, as illustrated in Figure~\ref{fig:main}. At a budget of 32 document vectors, this corresponds to approximately 32$\times$ fewer vectors than ColPali and 155$\times$ fewer vectors than ColQwen. Across the evaluated settings, \method is competitive with post-hoc compression and trainable compact-retrieval baselines, while its generated prefixes exhibit an empirical budget--quality scaling pattern. Because vector count alone does not determine byte-level storage, we additionally report embedding dimensions, representation size, index size, and retrieval latency.
\textbf{The key contributions of this paper are as follows:}
\begin{itemize}[noitemsep, topsep=0pt]
    \item[\ding{182}] We introduce a novel auto-regressive paradigm for visual document embedding that generates latent multi-vector representations sequentially, thereby departing from the conventional parallel patch-based encoding.
    \item[\ding{183}] Under stringent vector budgets, our method compares favorably with post-hoc clustering and the reproduced MetaEmbed and DocPruner baselines across the evaluated settings.
    \item[\ding{184}] We empirically identify a test-time budget--quality scaling behavior on two backbones, enabling retrieval quality to be adjusted by the number of generated vectors, with diminishing returns near the training budget.
    \item[\ding{185}] We quantify the distinction between vector-count compression and physical storage, together with the trade-off between reduced index/search cost and additional auto-regressive query-encoding latency.
\end{itemize}

%% file: content/related_work.tex
\section{Related Work}
\subsection{Multimodal Embedding}
Multimodal embedding focuses on representing heterogeneous inputs into a shared latent space for cross-modal understanding and retrieval~\citep{zhang2025mmretrieval_survey}. Built on the contrastive training paradigm proposed by~\citep{radford2021clip}, pre-trained models like CLIP~\citep{radford2021clip}, BLIP~\citep{li2022blip}, and SigLIP~\citep{zhai2023siglip} exhibit remarkable performance in multimodal encoding and modality alignment. Recently, pre-trained Multimodal Large Language Models (MLLMs) have also been adapted for multimodal embedding tasks, utilizing backbones that are more powerful for multimodal understanding. For instance, VLM2Vec~\citep{jiang2024vlm2vec} adapts Phi-3.5-V~\citep{abdin2024phi3}, LLaVE~\citep{lan2025llave} fine-tunes LLaVA-OneVision~\citep{li2024llavaov}, while VLM2Vec-V2~\citep{meng2025vlm2vec2}, jina-embeddings-v4~\citep{gunther2025jinav4}, and Nomic Embed Multimodal~\citep{nomicembedmultimodal2025} are built upon the Qwen-VL series~\citep{wang2024qwen2vl,bai2025qwen25vl,bai2025qwen3vl}. Beyond these advancements, recent studies have also discussed innovative strategies regarding data scheduling~\citep{nomicembedmultimodal2025,jian2025rzenembed} and training pipelines~\citep{li2025coma}, providing further insights into the dynamics of multimodal embedding training. Furthermore, combining embedding models with reasoning ability has been explored in text embedding~\citep{tsai2025giecse} and multimodal retrieval~\citep{lan2025ume,yu2025cafe,cui2025think}, an area that is still in its early stages.

\subsection{Visual Document Retrieval}
VDR has evolved significantly from traditional \textbf{OCR-based pipelines}, which \textit{often lose critical layout and structural information} during text extraction \citep{most2025lost,zhang2025ocr}.
The advent of MLLMs first prompted a shift towards \textbf{single-vector} representations, but these are inherently \textit{constrained in their ability to capture the fine-grained details} of content-dense pages \citep{ma2024unifying,zhang2024gme,liu2025any}. 
Consequently, \textbf{multi-vector} methods, epitomized by ColPali \citep{faysse2024colpali}, have become state-of-the-art, achieving superior performance by representing pages as hundreds or thousands of patch-level embeddings \citep{gunther2025jinav4,ops_colqwen3_4b,jayaram2024muvera,huang2025tomoro}. 
This approach, however, introduces a critical efficiency bottleneck due to \textit{prohibitive storage overhead} \citep{khattab2020colbert,santhanam2022colbertv2,cha2026reinpool}.
To mitigate this, recent research has explored three primary optimization paradigms: (i) \textbf{pruning} methods like DocPruner \citep{yan2025docpruner}, which aim to identify and discard redundant embeddings post-encoding; 
(ii) \textbf{merging} approaches such as Light-ColPali \citep{ma2025lightcolpali}, which employ clustering techniques to aggregate similar embeddings into a more compact set; and
(iii) strategies like MetaEmbed \citep{xiao2025metaembed}, which introduce a \textbf{fixed set of learnable tokens} whose final representations serve as a concise multi-vector set. 
While these methods focus on compressing or selecting from existing representations, they remain fundamentally \textit{limited by the initial encoding quality or a pre-defined token budget}.
In contrast, \method diverges from these compression-centric paradigms by reframing embedding creation as an auto-regressive generation process and synthesizing a compact set of latent vectors sequentially. We compare these alternatives under matched vector budgets and separately report retrieval quality, physical storage, and latency.
\subsection{Autoregressive Generation}\label{app:autoregressive_generation}
Autoregressive generation is the cornerstone paradigm for modern, high-performing LLMs~\citep{achiam2023gpt, team2023gemini, guo2025deepseek}. While this paradigm dominates natural language generation, several works have begun to explore its adoption in computer vision~\citep{sun2024autoregressive, tian2024var}. For example, LlamaGen~\citep{sun2024autoregressive} directly applies the original \textit{next-token} prediction paradigm of large language models to the visual generation domain. Meanwhile, VAR~\citep{tian2024var} introduces a coarse-to-fine \textit{next-scale (resolution)} prediction strategy, enabling efficient learning of visual distributions and strong generalization capabilities in autoregressive transformers.
Moreover, recent research has increasingly focused on building unified models that combine both understanding and generation~\citep{wu2025janus, chen2025janus, liu2025javisgpt}, facilitating end-to-end autoregressive generation for both visual and textual modalities. Beyond general generation tasks, some works have explored combining the generative paradigm with text and multimodal embeddings~\citep{cui2025think, tsai2025giecse, lan2025ume, liu2025reasoning, lin2025causal2vec, zhou2024vista}. For instance,~\citet{tsai2025giecse} iteratively refines single-vector text embeddings through latent reasoning, drawing inspiration from the latent reasoning paradigm introduced in COCONUT~\citep{hao2024coconut}.

%% file: content/method.tex
\section{Preliminary}
\par\textbf{Notations.}
We consider a multimodal embedding system $\mathcal{M}$ that maps a visual document page $I$ and a textual query $T$ into a shared latent space via a sequential generation process. Concretely, let $I \in \mathcal{I}$ denote a visually rendered document page, and let $T \in \mathcal{T}$ denote a tokenized text query. The system is defined as
\begin{equation}
\mathcal{M}=\langle f_{\theta}(\cdot),\, \mathcal{H},\, \mathcal{Z},\, \mathcal{S},\, \mathcal{T},\, \mathcal{I}\rangle,
\end{equation}
where $f_{\theta}$ denotes a vision--language model parameterized by $\theta$, consisting of a vision encoder $\Phi$ and an auto-regressive language model $\Psi$, i.e., $\theta \triangleq \{\Phi,\Psi\}$.
The shared latent space is denoted by $\mathcal{H}$, while $\mathcal{Z}\subset\mathbb{R}^{D}$ is an unconstrained embedding space of dimension $D$. Unless otherwise stated, CausalQwen and CausalPali use their original language-model hidden states directly as retrieval vectors; Section~\ref{sec:exp_eff} additionally evaluates a learned 128-dimensional projection.
The similarity function $\mathcal{S}$ specifies the scoring mechanism, $\mathcal{S}:\mathcal{T}\times\mathcal{I}\rightarrow \mathbb{R}$.
\par\textbf{Forward-based Multi-vector Embeddings.}
Following \citet{faysse2024colpali}, multi-vector embedding models typically leverage a pre-trained vision--language model to produce patch-/token-level representations for document pages and textual queries. Given a visual page $I$, the vision encoder $\Phi$ projects it into the latent space $\mathcal{H}$ as a sequence of visual features:
\begin{equation}
\mathbf{H}^{(v)}=\Phi(I)=\left[\mathbf{h}^{(v)}_{1},\ldots,\mathbf{h}^{(v)}_{N_v}\right],
\end{equation}
where $N_v$ is the length of the resulting visual feature sequence. These features are then fed into the language model $\Psi$ to produce the document embedding sequence:
\begin{equation}
\mathbf{D}=\Psi\!\left(\mathbf{H}^{(v)}\right)=\left[\mathbf{d}_{1},\ldots,\mathbf{d}_{N_v}\right].
\end{equation}
Similarly, for a textual query $T=[T_1,\ldots,T_{N_t}]$, the language model produces a query embedding sequence
\begin{equation}
\mathbf{Q}=\Psi(T)=\left[\mathbf{q}_{1},\ldots,\mathbf{q}_{N_t}\right],
\end{equation}
where $N_t$ is the number of query embedding vectors in the forward-based model. We denote the above processes using a unified forward function:
\begin{equation}
\mathbf{D}_{\mathfrak f}=f_{\theta}(I), \qquad \mathbf{Q}_{\mathfrak f}=f_{\theta}(T).
\end{equation}
\par\textbf{Late Interaction Mechanism.}
To capture fine-grained semantic alignment, the multi-vector embedding paradigm employs a late-interaction scoring function $\mathcal{S}(T,I)$, which aggregates local correspondences between the query and document embedding sequences. In particular, the max-sim aggregation is widely applied as 
\begin{equation}
\mathcal{S}(T,I)=\sum_{i=1}^{N_q}\max_{j=1}^{N_d}\left\langle \mathbf{Q}_{i},\,\mathbf{D}_{j}\right\rangle,
\end{equation}
where $\langle\cdot,\cdot\rangle$ denotes the dot product and $N_d$ is the number of document embedding vectors.
\par\textbf{Contrastive Alignment.}
We train the model such that relevant query-document pairs receive higher late-interaction scores than non-relevant pairs. Given an in-batch set $\mathcal{B}=\{(T_k,I_k)\}_{k=1}^{b}$, where $I_k$ is the ground-truth page for query $T_k$, we use a contrastive objective that compares the positive score $s_k^{+}\triangleq \mathcal{S}(T_k,I_k)$ against the hardest in-batch negative $s_k^{-}\triangleq \max_{l\in\{1,\ldots,b\},\,l\neq k}\mathcal{S}(T_k,I_l)$.
The training loss is then defined as
\begin{equation}\label{eq:loss_m}
\mathcal{L}_{m}=\frac{1}{b}\sum_{k=1}^{b}\log\left(1+\exp\left(s_k^{-}-s_k^{+}\right)\right).
\end{equation}
\section{\method: Compact Auto-regressive Multi-vector Generation}
\begin{figure*}[t]
\centering
\includegraphics[width=0.98\textwidth]{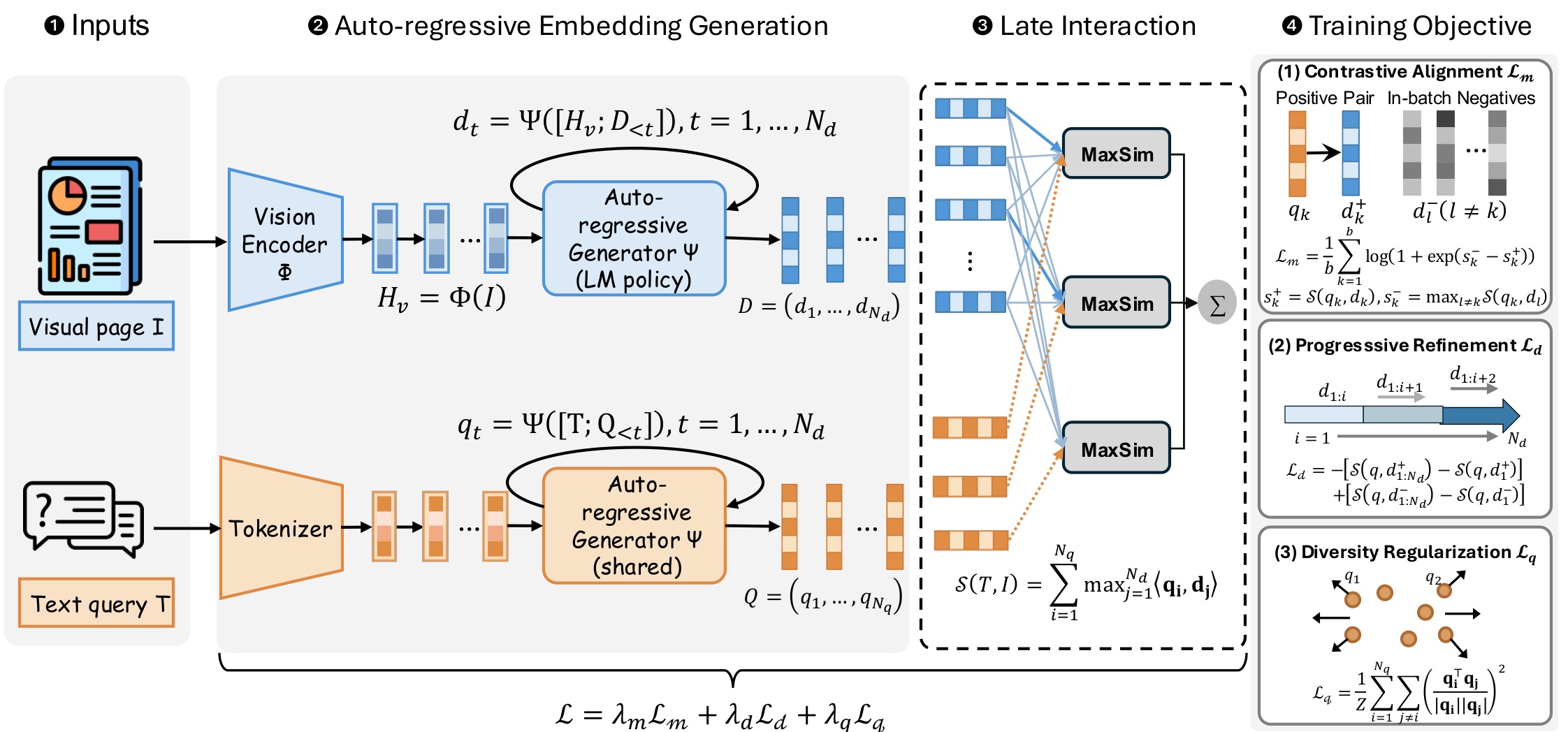}
\caption{Overview of our overall framework.}
\label{fig:frame}
\end{figure*}
\par\textbf{Auto-regressive Generative Embedding}\label{sec:auto}
The core of \method lies in transcending fixed-size representations by learning a dynamic generation policy.
Formally, let $\mathbf{C}$ denote the initial context embeddings derived from the input, where $\mathbf{C}=\mathbf{H}^{(v)}$ for visual pages and $\mathbf{C}=E(T)$ for textual queries, with $E$ denoting the original token-embedding layer.
We aim to generate a sequence of latent vectors $\mathbf{Z} = [\mathbf{z}_1, \ldots, \mathbf{z}_{L}]$, where $L \in \{N_d, N_q\}$ denotes the target embedding budget.
For each step $k \in \{1, \ldots, L\}$, the model $\Psi$ conditions on the original input and the history of previously generated latent states to produce the next embedding:
\begin{equation}
\mathbf{z}_k = \tau\left(\Psi\left(\left[\mathbf{C}, \mathbf{z}_1, \ldots, \mathbf{z}_{k-1}\right]\right)\right),
\end{equation}
where $[\cdot]$ denotes sequence concatenation along the temporal dimension, and $\tau(\cdot)$ extracts the final hidden state. Each generated state is appended directly to the input-embedding sequence for the next step; it does not pass through the discrete token-embedding layer. Its position index continues from the preceding context under the backbone's original positional-encoding scheme. During training, the cache is constructed through the official Transformers interface without detaching cached states, so gradients remain connected across generation steps; at inference, the same cache avoids recomputing the full prefix. We further denote this process using a unified generation function:
\begin{equation}
\mathbf{D}_{\mathfrak{g}} = g_{\theta}(I), \qquad \mathbf{Q}_{\mathfrak{g}} = g_{\theta}(T).
\end{equation}
\subsection{Rethinking Multi-vector Embeddings from the Perspective of Gradient Flow}\label{sec:grad}
\par\textbf{Problem Setup.}
We now delve deeper into the training dynamics of multi-vector embeddings through the lens of gradient flow. Consider a triplet sample $(T, I^+, I^-)$ with corresponding embedding sequences $(Q, D^+, D^-)$.
Let
\begin{equation}
j^*(i;Q,D) \triangleq \arg\max_{j}\langle q_i, d_j\rangle
\end{equation}
denote the index of the best-matched document token for query token $q_i$. The late-interaction score can be written as:
\begin{equation}
S(Q,D) = \sum_{i} q_i^{\top} d_{j^*(i;Q,D)}.
\end{equation}
The pairwise contrastive loss is defined as
$\mathcal{L} = \log\left(1 + \exp\left(S^- - S^+\right)\right)$,
where $S^+ = S(Q, D^+)$ and $S^- = S(Q, D^-)$.
Define the error signal as
\[
\delta \triangleq \frac{\partial \mathcal{L}}{\partial S^+}
= -\frac{\partial \mathcal{L}}{\partial S^-}
= \sigma(S^- - S^+) - 1,
\]
where $\sigma(\cdot)$ is the sigmoid function.
\par\textbf{Forward-based Patch-level Learning.}
Let $Q_{\mathfrak{f}} = \{q_1, \dots, q_{N_t}\}$ and $D_{\mathfrak{f}} = \{d_1, \dots, d_{N_v}\}$. For patch-level multi-vector embeddings, the gradient is:
\begin{equation}\label{eq:delta_f}
\nabla_{\theta}\mathcal{L}_{\mathfrak{f}} = \delta \cdot \left(\nabla_{\theta} S^+ - \nabla_{\theta} S^-\right),
\end{equation}
where
\begin{equation}\label{eq:delta_s}
\begin{aligned}
\nabla_{\theta} S(Q,D)
&= \sum_{i=1}^{N_t} \left( \frac{\partial q_i}{\partial \theta} \cdot d_{j^*(i;Q,D)} + \frac{\partial d_{j^*(i;Q,D)}}{\partial \theta} \cdot q_i \right).
\end{aligned}
\end{equation}
Here, $q_i$ and $d_j$ are generated in parallel by $f_{\theta}(\cdot)$.
Plugging Eq.~\ref{eq:delta_s} into Eq.~\ref{eq:delta_f}, we get:
\begin{equation}\label{eq:grad_f}
\begin{aligned}
\nabla_{\theta}\mathcal{L}_{\mathfrak{f}} = \delta \cdot \sum_{i=1}^{N_t}\Big(
&\underbrace{\frac{\partial q_i}{\partial \theta} \cdot \left(d^{+}_{j^*(i;Q,D^+)} - d^{-}_{k^*(i;Q,D^-)}\right)}_{\text{Query update } \mathfrak{Q}} + &\underbrace{\frac{\partial d^{+}_{j^*(i;Q,D^+)}}{\partial \theta} \cdot q_i}_{\text{PosDoc update } \mathfrak{D}^{+}} 
- \underbrace{\frac{\partial d^{-}_{k^*(i;Q,D^-)}}{\partial \theta} \cdot q_i}_{\text{NegDoc update } \mathfrak{D}^{-}} 
\Big),
\end{aligned}
\end{equation}
where $k^*(i;Q,D^-) \triangleq \arg\max_{j} \langle q_i, d^{-}_j \rangle$, and $\frac{\partial d}{\partial \theta}$ denotes the Jacobian.
This decomposition suggests that each query token contributes to parameter updates via $\partial q_i/\partial \theta$, while only the MaxSim-selected document outputs receive direct supervision on the document side. When $N_t \ll N_v$, this direct signal is sparse over the document outputs. Appendix~\ref{app:comp} provides a stylized dependency-path count; it does not by itself establish faster optimization or lower wall-clock training cost.
\par\textbf{Auto-regressive Multi-vector Learning.}
Let $Q_{\mathfrak{g}} = \{q_{1}, \dots, q_{N_q}\}$ and $D_{\mathfrak{g}} = \{d_{1}, \dots, d_{N_d}\}$ denote the generated query and document multi-vectors from \method.
To account for the recursive dependency introduced by auto-regressive generation, we adopt a gradient backpropagation view:
\begin{equation}\label{eq:bptt}
\frac{d d_t}{d \theta}
= \frac{\partial d_t}{\partial \theta}
+ \sum_{k < t} \frac{\partial d_t}{\partial d_k} \cdot \frac{d d_k}{d \theta},
\end{equation}
where $\frac{\partial d_t}{\partial d_k}$ captures historical dependencies. Then:
\begin{equation}
\begin{aligned}
\nabla_{\theta}\mathcal{L}_{\mathfrak{g}}
&= \delta \cdot \sum_{i=1}^{N_q} \Bigg[
\left(
\frac{\partial q_i}{\partial \theta}
+ \sum_{\tau < i} \frac{\partial q_i}{\partial q_{\tau}} \cdot \frac{d q_{\tau}}{d\theta}
\right) \cdot \Delta d_i + q_i^{\top} \cdot \frac{d d^{+}_{j^*(i;Q,D^+)}}{d\theta}
- q_i^{\top} \cdot \frac{d d^{-}_{k^*(i;Q,D^-)}}{d\theta}
\Bigg],
\end{aligned}
\end{equation}
where $\Delta d_i \triangleq d^{+}_{j^*(i;Q,D^+)} - d^{-}_{k^*(i;Q,D^-)}$.
Compared with Eq.~\ref{eq:grad_f}, \method makes each generated token depend on the preceding latent context, creating additional gradient paths through the generation chain. The controlled architecture experiment in Appendix~\ref{app:controlled_prefix} evaluates whether this structural difference translates into more useful prefixes. Further analytical details and limitations are provided in Appendix~\ref{app:analysis}.
\subsection{Overall Training Objective}\label{sec:obj}
Although the analysis in Section~\ref{sec:grad} highlights the additional dependencies introduced by auto-regressive multi-vector learning, directly applying Eq.~\ref{eq:loss_m} can lead to homogeneous generated representations. To address this, we introduce document-side progressive refinement and query-side diversity regularization.
\noindent\par\textbf{\ding{171} Progressive Refinement Loss ($\mathcal{L}_d$)} encourages the complete document sequence to improve the positive--negative separation relative to its first generated vector:
\begin{equation}
\mathcal{L}_d = -\left[ \mathcal{S}(q, d^+_{1:N_d}) - \mathcal{S}(q, d^+_{1}) \right] + \left[ \mathcal{S}(q, d^-_{1:N_d}) - \mathcal{S}(q, d^-_{1}) \right].
\end{equation}
This endpoint objective encourages the model to utilize the sequence beyond its first vector. It does not explicitly impose a monotonic constraint on every intermediate prefix; we examine the resulting prefix behavior and an explicitly prefix-supervised alternative in Appendix~\ref{app:controlled_prefix}.
\noindent\par\textbf{\ding{168} Diversity Regularization ($\mathcal{L}_q$)} serves as an anti-homogenization constraint. To prevent the auto-regressive generator from collapsing into repetitive patterns, we penalize the cosine similarity between distinct tokens in the generated query sequence:
\begin{equation}
\mathcal{L}_q = \frac{1}{Z} \sum_{i=1}^{N_q} \sum_{j \neq i} \left( \frac{\mathbf{q}_i^\top \mathbf{q}_j}{\|\mathbf{q}_i\| \|\mathbf{q}_j\|} \right)^2,
\end{equation}
where $Z$ is the total number of off-diagonal token pairs.
We apply this term only to queries: document tokens already receive the task-aligned $\mathcal{L}_d$, and adding the same pairwise diversity constraint on documents reduces full-budget performance in the controlled ablation of Appendix~\ref{app:regularization_ablation}.
\noindent\par\textbf{Final Objective.}
The total training loss of \method combines the components as:
\begin{equation}
\label{eq:loss}
\mathcal{L} = \lambda_{m}\mathcal{L}_m + \lambda_{d}\mathcal{L}_d + \lambda_{q}\mathcal{L}_q,
\end{equation}
where we use $\lambda_{m} = 1$ and $\lambda_{d} = \lambda_{q} = 0.1$ by default throughout our experiments.

%% file: content/experiment.tex
\section{Experiment}
\subsection{Experiment Setup}
\input{tables/tab_v2_3}
\par\textbf{Training Details} To evaluate \method across two model architectures, we use PaliGemma-3B~\citep{beyer2024paligemma} and Qwen2.5-VL-3B~\citep{bai2025qwen25vl}. Following ColPali~\citep{faysse2024colpali}, we train on the same training set\footnote{Training set: \href{https://huggingface.co/datasets/vidore/colpali_train_set}{vidore/colpali\_train\_set}. Pre-trained BiEncoders: \href{https://huggingface.co/vidore/biqwen2-v0.1}{vidore/biqwen2-v0.1} and \href{https://huggingface.co/vidore/bipali}{vidore/bipali}.}. Additional training and implementation details are provided in Appendix~\ref{app:train}.
\par\textbf{Baselines} We compare with (i) post-hoc random selection, K-Means~\citep{mcqueen1967kmeans}, semantic clustering, and 1D pooling applied to frozen ColQwen2.5 or ColPali representations; (ii) the trainable single-vector BiQwen2 and BiPali checkpoints; and (iii) reproduced trainable compression methods, MetaEmbed~\citep{xiao2025metaembed} and DocPruner~\citep{yan2025docpruner}. Because the post-hoc methods do not receive the LoRA training used by \method, we additionally conduct a controlled auto-regressive-versus-parallel comparison with matched initialization, data, objective, and one-epoch fine-tuning budget (Appendix~\ref{app:controlled_prefix}). The MetaEmbed reproduction uses the Qwen2.5-VL-3B family, the ColPali training set, LoRA rank 32, and the same global batch size of 32; complete settings and results are reported in Appendix~\ref{app:learned_baselines}.
\par\textbf{Evaluation} We assess the general multimodal embedding capability of \method using three versions of the \vidore benchmark~\citep{faysse2024colpali, mace2025vidore2, loison2026vidore3}, and report standard nDCG@5 as the evaluation metric. \vidore~is widely adopted to benchmark visual document retrieval across various domains. Notably, \vidore~V2 addresses performance saturation by introducing more generalized settings and multilingual subsets, while the latest \vidore~V3 evaluates the model's ability to retrieve accurate information from complex, visually rich documents across diverse industrial contexts.
\subsection{Main Results}
\input{tables/tab_v1}
Tables~\ref{tab:\vidore_v2&3} and \ref{tab:\vidore_v1} present a comparison of various methods across \vidore~V1 to V3. Our key findings are as follows:
\par\textbf{Post-hoc Compression Incurs Retrieval Loss at 32 Vectors.} As shown in Table~\ref{tab:\vidore_v2&3}, compressing the base representations to 32 vectors generally reduces performance on visually rich retrieval tasks. On \vidore~V2, the best 32-vector post-hoc baselines remain below their full-vector ColQwen2.5 and ColPali counterparts; the same pattern appears on \vidore~V3. Table~\ref{tab:\vidore_v1} shows a smaller gap for the Qwen-based clustering methods.
\begin{wrapfigure}{r}{0.65\textwidth}
\centering
\includegraphics[width=0.64\textwidth]{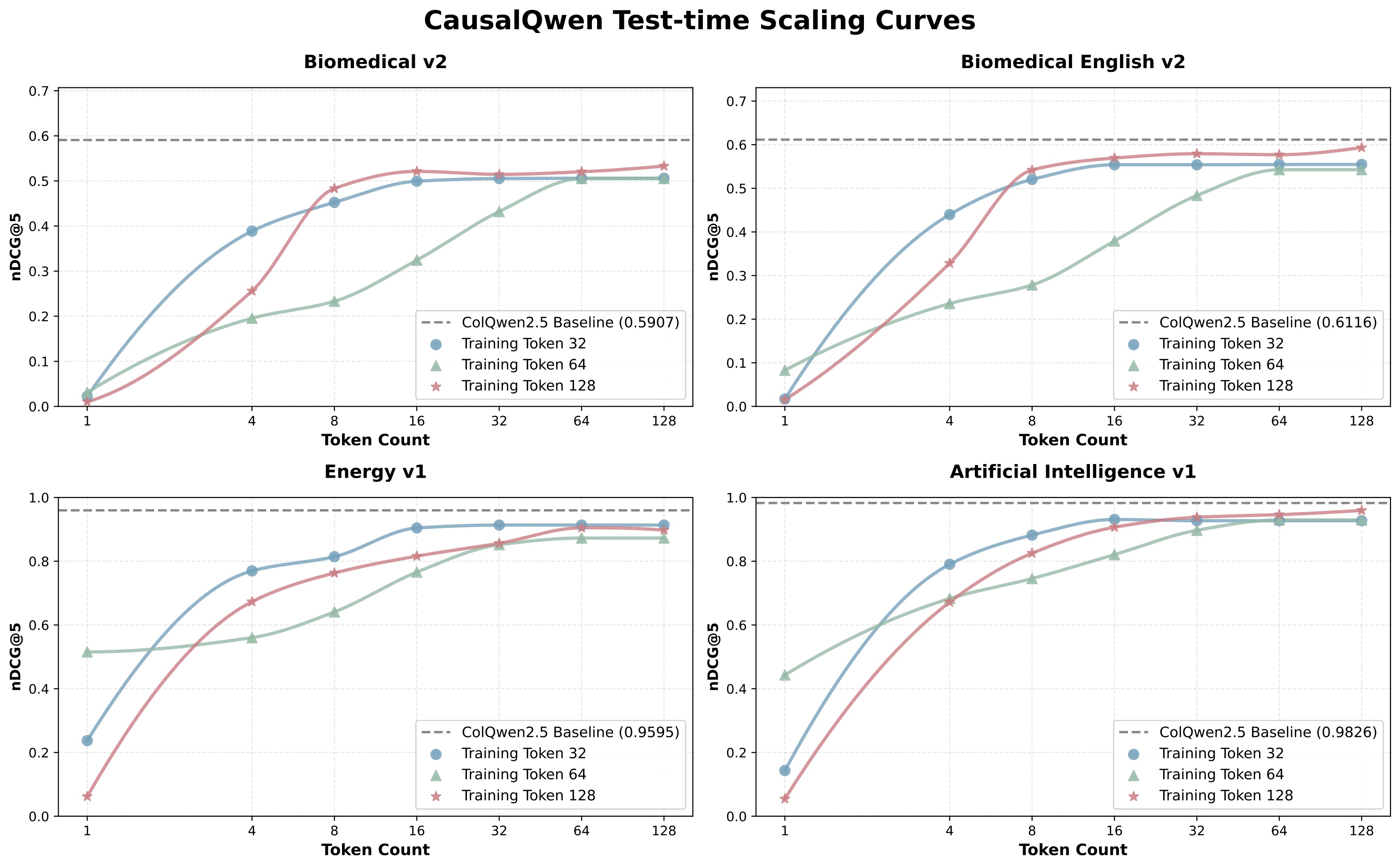}
\caption{Test-time scaling curves of \method. We train CausalQwen with sequence lengths of 32, 64, and 128, and evaluate its test-time performance under different inference budgets.}
\label{fig:scaling}
\vspace{-4mm}
\end{wrapfigure}
\par\textbf{\method Is Competitive under a Stringent Vector Budget.} On \vidore~V2 and V3, CausalQwen obtains average nDCG@5 scores of 0.516 and 0.426, exceeding the best listed 32-vector post-hoc baselines by 0.061 and 0.024, respectively. On \vidore~V1, it obtains 0.811, slightly below K-Means (0.816) and SemCluster (0.813). The result therefore supports competitiveness, rather than uniform superiority, across all benchmarks. For the PaliGemma architecture, CausalPali exceeds the full-vector ColPali average on V1, V2, and V3 by 0.045, 0.057, and 0.048, respectively; this full-model improvement is specific to the weaker PaliGemma-based system and is not observed for CausalQwen.
\par\textbf{Comparison with Trainable Compression.} In the matched Qwen-family reproduction summarized in Appendix~\ref{app:learned_baselines}, CausalQwen reaches 0.60915 mean nDCG@5 over the 25-subset evaluation, compared with 0.51755 at the best MetaEmbed operating point. On the six-subset DocPruner evaluation, 32-vector CausalQwen obtains 0.49722, whereas DocPruner obtains 0.29742, 0.39778, and 0.48404 with 32, 64, and 128 document vectors. Together with the controlled architecture comparison, these results show that the gain is not explained solely by comparing a trained model with training-free clustering.
\par\textbf{Evidence across the Evaluated Backbones.} The results cover two distinct 3B architectures and three benchmark versions. Both CausalQwen and CausalPali outperform the corresponding 32-vector random and 1D-pooling baselines throughout the reported benchmark averages, while the relative ranking against clustering depends on the backbone and benchmark. We therefore limit the cross-architecture claim to these two evaluated backbones rather than asserting scale-wide or distribution-independent generalization.
\subsection{Test-time Scaling of \method}\label{sec:tts}
To investigate the influence of latent sequence length, we evaluate document prefixes with $N_d\in\{1,2,4,8,16,24,32\}$ on six subsets. Mean nDCG@5 for CausalQwen increases from 0.05045 at one vector to 0.49648 at 24 and 0.49722 at 32; CausalPali increases from 0.02896 to 0.41040 and 0.41880. Thus, both evaluated backbones exhibit an empirical budget--quality scaling trend with diminishing returns near 32 vectors. Individual subsets contain small local dips near saturation, so we do not claim strict token-by-token monotonicity. Complete per-subset results are provided in Appendix~\ref{app:prefix_scaling}.
The prefix behavior is reminiscent of \textbf{Matryoshka Representation}: shorter prefixes provide lower-cost representations, while longer prefixes generally refine retrieval quality. However, the controlled results in Appendix~\ref{app:controlled_prefix} show that this behavior cannot be attributed to a single component. Under the same main objective, auto-regressive generation improves prefix-AUC from 0.42309 to 0.47807 relative to a parallel variant, while explicit prefix supervision also produces useful parallel prefixes. We therefore describe the result as an observed property of the trained \method models, rather than an inherent property of all auto-regressive embeddings.
\subsection{Framework Analysis}\label{sec:exp_eff}
\begin{figure*}[t]
    \centering
    \begin{minipage}{0.98\linewidth}
        \centering
\includegraphics[width=\linewidth]{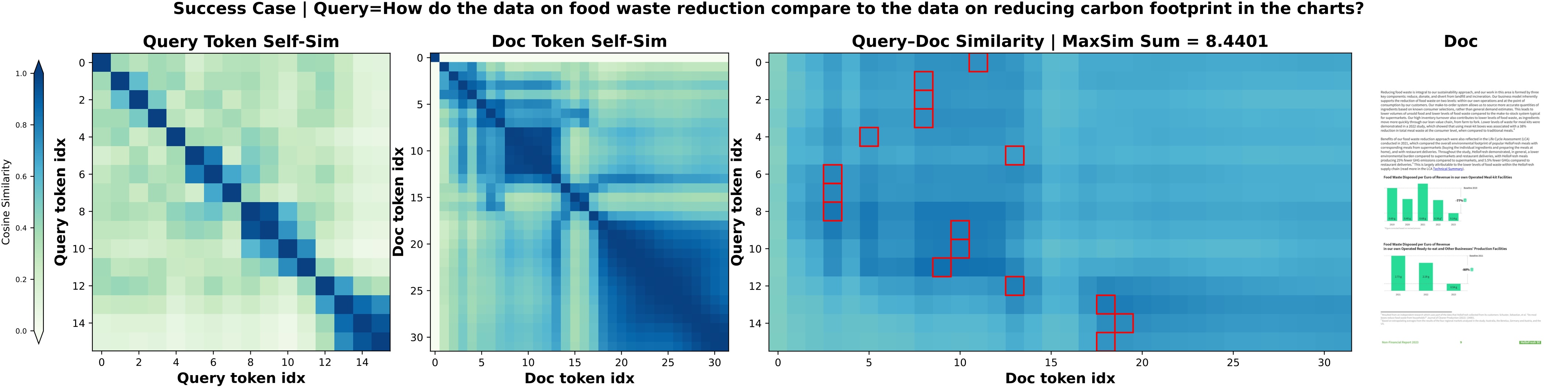}
\subcaption{Visualization of a successful case. For each query token, document token with highest score is highlighted with a red box.}
        \label{fig:success}
    \end{minipage}
    \vfill
    \begin{minipage}{0.98\linewidth}
        \centering \includegraphics[width=\linewidth]{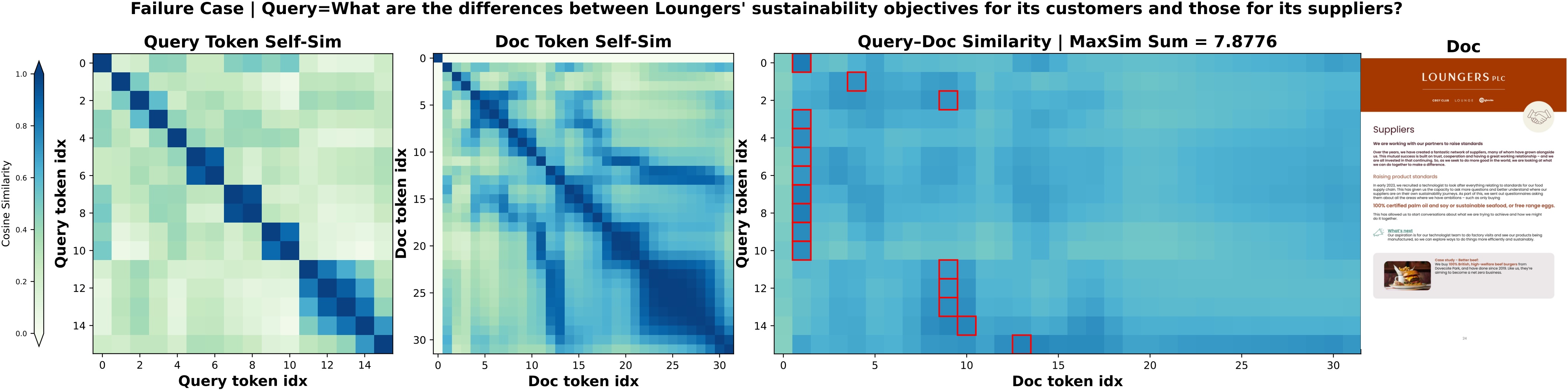}
\subcaption{Visualization of a failure case. For each query token, document token with highest score is highlighted with a red box.}
        \label{fig:fail}
    \end{minipage}
    \caption{Case study of CausalQwen on \vidore~V2. Heatmaps visualize the self-similarity within query and document sequences, as well as their cross-similarity.}
    \label{fig:viz}
\end{figure*}
\par\textbf{Storage and Latency Trade-off.}
Vector count, byte-level storage, and online latency are distinct. ColQwen represents a page with $4{,}962\times128$ values, whereas the main CausalQwen uses $32\times2{,}048$ values; this is a 155$\times$ reduction in vector count but a 9.7$\times$ reduction in FP16 representation bytes. An additional learned 128-dimensional projection yields $32\times128$ values (8,192 bytes per page), retaining 93.90\% of the six-subset mean nDCG@5 of the 2,048-dimensional model and providing a 155.1$\times$ byte reduction in the controlled 10K-representation setting.
Figure~\ref{fig:bubble} reports the local adaptation-latency/performance trade-off on \vidore~V2. The full system profile in Appendix~\ref{app:latency_breakdown} gives a more complete picture: compact document representations reduce MaxSim and index storage, but query-side auto-regressive encoding increases online latency. At a 100K corpus size, CausalQwen-128D reduces measured search time from 15.947 to 0.885 ms and index size from 19.719 to 0.819 GB, while encode-plus-search time rises from 37.557 to 335.860 ms. Thus, the current implementation is attractive when storage or repeated search dominates, but it is not universally faster end-to-end.
\begin{figure}[t]
\centering
\includegraphics[width=0.65\linewidth]{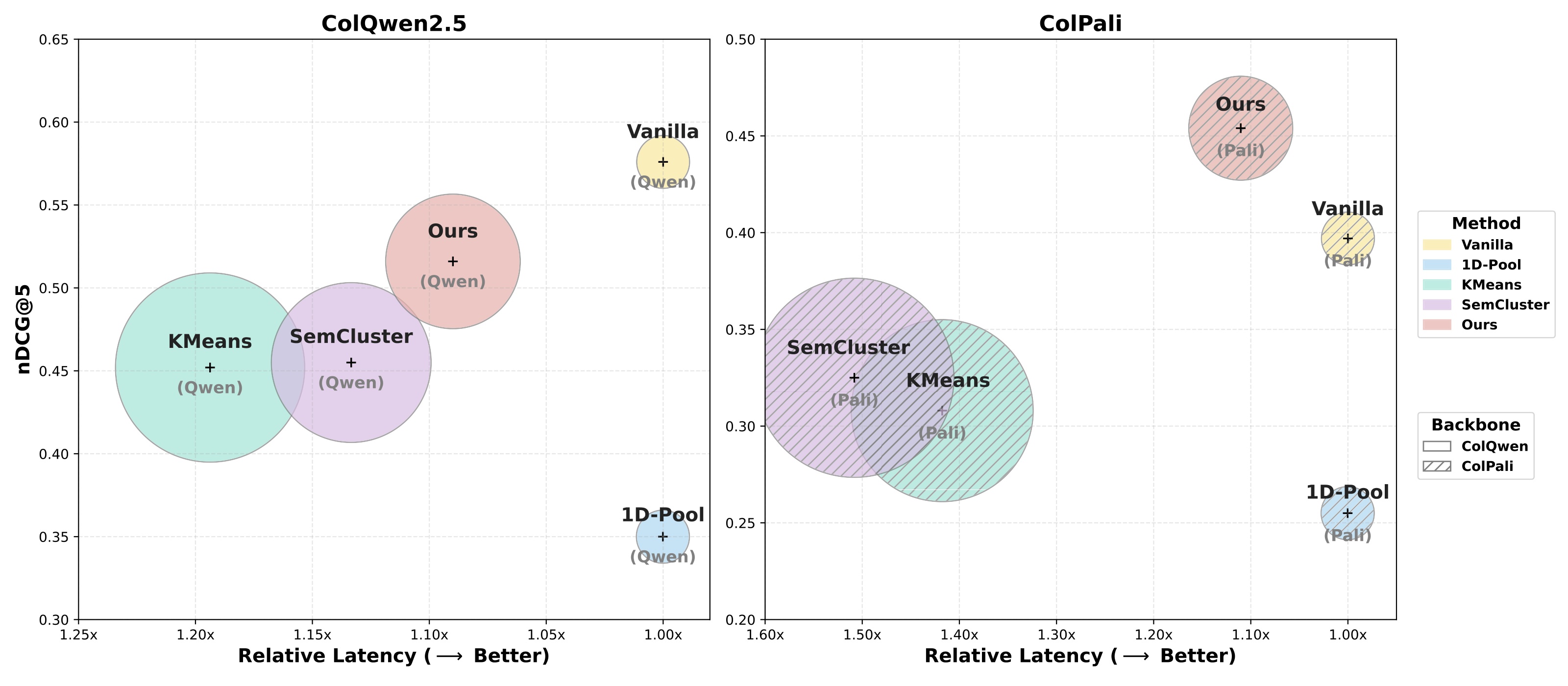}
\caption{Trade-off between retrieval performance and local representation-adaptation latency on \vidore~V2. Bubble size indicates adaptation overhead; it does not include the complete online retrieval path.}
\label{fig:bubble}
\end{figure}
\par\textbf{Visualization.}
Figure~\ref{fig:viz} illustrates the decision-making process of CausalQwen on \vidore~V2. Consistent with the objectives described in Section~\ref{sec:obj}, CausalQwen produces diverse yet compact token-wise embeddings for both queries and documents. Notably, most document information is concentrated in the early tokens, which contribute the most to the final similarity score, while later tokens progressively capture finer-grained details, as reflected by the distribution of highlighted document tokens.\par
\subsection{Ablation Study}
\begin{table}[h]
\centering
\caption{Ablation study on query-side autoregressive generation (NDCG@5). \textit{Query-Forward + Doc-AR} uses standard forward-pass encoding for queries and autoregressive generation only for documents. $\Delta$ denotes the performance change compared to CausalQwen (Ours). Negative values indicate performance degradation.}
\label{tab:query_ablation}
\renewcommand{\arraystretch}{1.1}
\resizebox{\linewidth}{!}{
\begin{tabular}{l|cc|cc|cc|cc|cc}
\toprule
\textbf{Method} & \textbf{ArxivQA} & $\Delta$ & \textbf{DocVQA} & $\Delta$ & \textbf{InfoVQA} & $\Delta$ & \textbf{TabFQuad} & $\Delta$ & \textbf{TAT-DQA} & $\Delta$ \\
\midrule
CausalQwen (Ours) & 0.807 & -- & 0.553 & -- & 0.857 & -- & 0.825 & -- & 0.620 & -- \\
Query-Forward + Doc-AR & 0.702 & \textcolor{red}{-0.105} & 0.472 & \textcolor{red}{-0.081} & 0.794 & \textcolor{red}{-0.063} & 0.707 & \textcolor{red}{-0.118} & 0.505 & \textcolor{red}{-0.115} \\
\midrule
\textbf{Method} & \textbf{ShiftProject} & $\Delta$ & \textbf{Syn-AI} & $\Delta$ & \textbf{Syn-Energy} & $\Delta$ & \textbf{Syn-Gov} & $\Delta$ & \textbf{Syn-Health} & $\Delta$ \\
\midrule
CausalQwen (Ours) & 0.737 & -- & 0.927 & -- & 0.907 & -- & 0.935 & -- & 0.941 & -- \\
Query-Forward + Doc-AR & 0.440 & \textcolor{red}{-0.297} & 0.856 & \textcolor{red}{-0.071} & 0.776 & \textcolor{red}{-0.131} & 0.871 & \textcolor{red}{-0.064} & 0.913 & \textcolor{red}{-0.028} \\
\bottomrule
\end{tabular}
}
\end{table}
\begin{wraptable}{r}{0.6\textwidth}
\centering
\caption{Ablation study of four variants, each modifying a key component of CausalQwen.}
\label{tab:ablation}
\renewcommand{\arraystretch}{1.3}
\resizebox{0.98\linewidth}{!}{
\begin{tabular}{l|cc}
\hline
\par\textbf{Method Variant} & \par\textbf{\vidore~V1} & \par\textbf{\vidore~V2} \\
\hline
\rowcolor{tableblue!25} 
Ours (Full Model) & \par\textbf{0.811} & \par\textbf{0.516} \\
\midrule
w/o $\mathcal{L}_d$ & 0.697\,\red{14.1\%} & 0.436\,\red{15.5\%} \\
w/o $\mathcal{L}_q$ & 0.695\,\red{14.3\%} & 0.480\,\red{7.0\%} \\
w/o $\mathcal{L}_m$ & 0.004\,\red{99.5\%} & 0.012\,\red{97.7\%} \\
w/ $\text{ReLU}(\Delta \mathcal{S})$ & 0.633\,\red{21.9\%} & 0.365\,\red{29.3\%} \\
\hline
\end{tabular}
}
\end{wraptable}
\textbf{Query-Side Autoregressive Ablation.}
\label{app:query_ablation}
To investigate the necessity of query-side autoregressive generation, we conduct an ablation study comparing our full method (CausalQwen) with a variant that uses standard forward-pass encoding for queries while maintaining autoregressive generation exclusively for documents. Table~\ref{tab:query_ablation} presents the results on ViDoRe V1 datasets.
As shown in Table~\ref{tab:query_ablation}, replacing query-side auto-regressive generation with the forward representation degrades performance across these datasets. The largest drops occur on ShiftProject ($-0.297$), Syn-Energy ($-0.131$), and TabFQuad ($-0.118$). This result indicates that query-side generation contributes materially in the evaluated configuration; it does not by itself identify a unique causal mechanism for the degradation.\par
\textbf{Loss Function Ablation.}
We conduct an ablation study on \vidore~V1 and \vidore~V2 to assess the contribution of each objective term in Eq.~\ref{eq:loss}, with results summarized in Table~\ref{tab:ablation}. Removing $\mathcal{L}_m$ leads to an almost complete performance collapse on both benchmarks, highlighting its indispensable role in stabilizing training and establishing a meaningful embedding space. Disabling either the document-side objective $\mathcal{L}_d$ or the query-side objective $\mathcal{L}_q$ also results in consistent performance degradation.
This indicates that encouraging gains from the complete document trajectory and mitigating representation homogenization can be important for effective auto-regressive embedding generation. Additional controlled analyses are provided in Appendix~\ref{app:exp}.

%% file: tables/tab_v2_3.tex
\newcommand{\up}{\rlap{\textsuperscript{\textcolor{mydarkred}{$\uparrow$}}}}
\begin{table*}[t!]
\centering
\caption{nDCG@5 of different methods on the \vidore~V3 and V2 benchmarks. \textbf{Bold} and \underline{underlined} denote the best and second-best results among compact methods. \colorbox{tableyellow}{Yellow} and \colorbox{tableblue}{blue} cells represent trainable and our proposed methods, respectively. \textcolor{red}{$\uparrow$} indicates higher values are better, while \textsuperscript{\textcolor{mydarkred}{$\uparrow$}} marks performance above the corresponding full-vector base model. Token count and embedding dimension are reported separately in Section~\ref{sec:exp_eff}.}
\begin{adjustbox}{width=0.98\textwidth,center}
\begin{tabular}{l|c|ccccccccc|ccccc}
\toprule
\multirow{2}{*}{\textbf{Method}} & \multirow{2}{*}{\textbf{Doc. Vec.}}
& \multicolumn{9}{c|}{\textbf{\vidore~V3}} & \multicolumn{5}{c}{\textbf{\vidore~V2}} \\
\cline{3-16}
& & \textbf{HR} & \textbf{Fin-E} & \textbf{Ind.} & \textbf{Phar.} & \textbf{C.S.} & \textbf{Ener.} & \textbf{Phys.} & \multicolumn{1}{c|}{\textbf{Fin-F}} & \textbf{Avg (\textcolor{red}{$\uparrow$})} & \textbf{ESG} & \textbf{Bio} & \textbf{Econ} & \multicolumn{1}{c|}{\textbf{ESG-H}} & \textbf{Avg (\textcolor{red}{$\uparrow$})} \\
\midrule
\multicolumn{16}{c}{\textbf{ColQwen2.5}} \\
\midrule
\rowcolor{black!10}\cellcolor{tableyellow}Base & 4962 & 0.473 & 0.500 & 0.416 & 0.561 & 0.686 & 0.571 & 0.434 & 0.375 & 0.502 & 0.549 & 0.591 & 0.544 & 0.620 & 0.576 \\
Random & 32 & 0.248 & 0.250 & 0.203 & 0.384 & 0.492 & 0.348 & 0.338 & 0.144 & 0.301 & 0.293 & 0.425 & 0.354 & 0.304 & 0.344 \\
SemCluster & 32 & 0.359 & \underline{0.366} & \underline{0.307} & \underline{0.481} & \underline{0.594} & 0.453 & \textbf{0.390} & \underline{0.267} & \underline{0.402} & \underline{0.426} & \textbf{0.515} & 0.485 & 0.395 & 0.455 \\
K-Means & 32 & \underline{0.365} & 0.360 & 0.309 & 0.467 & 0.579 & \textbf{0.464} & 0.383 & 0.247 & 0.397 & 0.413 & 0.504 & \textbf{0.530} & 0.359 & 0.452 \\
1D-Pooling & 32 & 0.260 & 0.283 & 0.237 & 0.404 & 0.561 & 0.397 & 0.379 & 0.186 & 0.338 & 0.251 & 0.472 & 0.417 & 0.261 & 0.350 \\
\cellcolor{tableyellow}BiQwen2 & 1 & 0.336 & 0.333 & 0.236 & 0.460 & 0.550 & 0.399 & 0.383 & 0.200 & 0.362 & 0.413 & 0.487 & 0.493 & \underline{0.504} & \underline{0.464} \\
\rowcolor{black!20} \cellcolor{tableblue}CausalQwen & 32 & \textbf{0.429} & \textbf{0.395} & \textbf{0.327} & \textbf{0.483} & \textbf{0.643} & \underline{0.460} & \underline{0.390} & \textbf{0.285} & \textbf{0.426} & \textbf{0.498} & \underline{0.505} & \underline{0.522} & \textbf{0.537} & \textbf{0.516} \\
\midrule
\multicolumn{16}{c}{\textbf{ColPali}} \\
\midrule
\rowcolor{black!10}\cellcolor{tableyellow}Base & 1031 & 0.321 & 0.232 & 0.228 & 0.391 & 0.488 & 0.248 & 0.248 & 0.144 & 0.288 & 0.350 & 0.433 & 0.397 & 0.406 & 0.397 \\
Random & 32 & 0.161 & 0.127 & 0.119 & 0.239 & 0.301 & 0.148 & 0.152 & 0.060 & 0.163 & 0.230 & 0.272 & 0.337 & 0.191 & 0.258 \\
SemCluster & 32 & \underline{0.251} & \underline{0.165} & \underline{0.164} & \underline{0.319} & \underline{0.390} & \underline{0.236} & 0.188 & \underline{0.102} & \underline{0.227} & 0.291 & \underline{0.366} & \underline{0.371} & \underline{0.271} & \underline{0.325} \\
K-Means & 32 & 0.230 & 0.163 & 0.158 & 0.308 & 0.380 & 0.233 & \underline{0.194} & 0.100 & 0.221 & \underline{0.299} & 0.354 & 0.334 & 0.243 & 0.308 \\
1D-Pooling & 32 & 0.204 & 0.108 & 0.125 & 0.223 & 0.339 & 0.111 & 0.134 & 0.053 & 0.162 & 0.163 & 0.286 & 0.353 & 0.216 & 0.255 \\
\cellcolor{tableyellow}BiPali & 1 & 0.197 & 0.104 & 0.100 & 0.182 & 0.372 & 0.150 & 0.175 & 0.071 & 0.169 & 0.155 & 0.259 & 0.369 & 0.216 & 0.261 \\
\rowcolor{black!20} \cellcolor{tableblue}CausalPali & 32 & \textbf{0.353}\up & \textbf{0.280}\up & \textbf{0.234}\up & \textbf{0.436}\up & \textbf{0.532}\up & \textbf{0.353}\up & \textbf{0.332}\up & \textbf{0.169}\up & \textbf{0.336}\up & \textbf{0.357}\up & \textbf{0.493}\up & \textbf{0.503}\up & \textbf{0.464}\up & \textbf{0.454}\up \\
\bottomrule
\end{tabular}
\end{adjustbox}
\label{tab:\vidore_v2&3}
\end{table*}

%% file: tables/tab_v1.tex
\begin{table*}[t!]
\small
\centering
\caption{nDCG@5 of different methods on the \vidore~V1 benchmark. \textbf{Bold} and \underline{underlined} denote the best and second-best results among compact methods. CausalQwen is competitive but does not uniformly outperform the clustering baselines on V1.
}
\begin{adjustbox}{width=0.98\textwidth,center}
\begin{tabular}{l|c|cccccccccc|c}
\toprule
\multirow{2}{*}{\textbf{Method}} & \multirow{2}{*}{\textbf{Doc. Vec.}}
& \multicolumn{10}{c|}{\textbf{\vidore~V1}} & \multirow{2}{*}{\textbf{Avg}(\textcolor{red}{$\uparrow$})} \\
\cline{3-12}
& & \textbf{Arxiv} & \textbf{Doc} & \textbf{Info} & \textbf{TabF} & \textbf{TatD} & \textbf{Shift} & \textbf{Syn-AI} & \textbf{Syn-En} & \textbf{Syn-GR} & \textbf{Syn-HI} & \\
\midrule
\multicolumn{13}{c}{\textbf{ColQwen2.5}} \\
\midrule
\rowcolor{black!10}\cellcolor{tableyellow}Base & 5046 & 0.876 & 0.622 & 0.931 & 0.874 & 0.805 & 0.847 & 0.983 & 0.960 & 0.955 & 0.993 & 0.885 \\
Random & 32 & 0.721 & 0.414 & 0.744 & 0.849 & \underline{0.685} & 0.488 & 0.854 & 0.867 & 0.838 & 0.844 & 0.730 \\
SemCluster & 32 & 0.832 & \underline{0.529} & 0.851 & \textbf{0.877} & \textbf{0.705} & 0.654 & \underline{0.933} & \underline{0.914} & 0.902 & 0.938 & \underline{0.813} \\
K-Means & 32 & \textbf{0.837} & 0.524 & \textbf{0.861} & \underline{0.866} & 0.672 & \underline{0.725} & 0.932 & \textbf{0.921} & 0.889 & 0.928 & \textbf{0.816} \\
1D-Pooling & 32 & 0.730 & 0.432 & 0.829 & 0.791 & 0.614 & 0.625 & 0.883 & 0.873 & 0.873 & 0.931 & 0.758 \\
\cellcolor{tableyellow}BiQwen2 & 1 & \underline{0.833} & 0.516 & 0.829 & 0.833 & 0.661 & 0.723 & \textbf{0.940} & 0.861 & \underline{0.925} & \textbf{0.952} & 0.807 \\ 
\rowcolor{black!20} \cellcolor{tableblue}CausalQwen & 32 & 0.807 & \textbf{0.553} & \underline{0.857} & 0.825 & 0.620 & \textbf{0.737} & 0.927 & 0.907 & \textbf{0.935} & \underline{0.941} & 0.811 \\
\midrule
\multicolumn{13}{c}{\textbf{ColPali}} \\
\midrule
\rowcolor{black!10}\cellcolor{tableyellow}Base & 1031 & 0.586 & 0.506 & 0.735 & 0.710 & 0.546 & 0.467 & 0.923 & 0.826 & 0.861 & 0.894 & 0.705 \\
Random & 32 & 0.464 & 0.280 & 0.577 & 0.656 & 0.435 & 0.187 & 0.791 & 0.687 & 0.677 & 0.779 & 0.554 \\
SemCluster & 32 & \underline{0.540} & 0.410 & \underline{0.668} & 0.663 & \underline{0.463} & 0.367 & \underline{0.855} & \underline{0.740} & \underline{0.804} & 0.797 & \underline{0.631} \\
K-Means & 32 & 0.537 & \underline{0.412} & 0.652 & \underline{0.681} & 0.450 & \underline{0.382} & 0.762 & 0.692 & 0.769 & \underline{0.823} & 0.616 \\
1D-Pooling & 32 & 0.461 & 0.288 & 0.598 & 0.526 & 0.360 & 0.205 & 0.803 & 0.646 & 0.784 & 0.731 & 0.540 \\
\cellcolor{tableyellow}BiPali & 1 & 0.422 & 0.282 & 0.627 & 0.669 & 0.284 & 0.288 & 0.618 & 0.553 & 0.602 & 0.624 & 0.497 \\ 
\rowcolor{black!20} \cellcolor{tableblue}CausalPali & 32 & \textbf{0.746}\up & \textbf{0.492} & \textbf{0.787}\up & \textbf{0.832}\up & \textbf{0.520} & \textbf{0.614}\up & \textbf{0.871} & \textbf{0.851}\up & \textbf{0.904}\up & \textbf{0.888} & \textbf{0.750}\up \\
\bottomrule
\end{tabular}
\end{adjustbox}
\label{tab:\vidore_v1}
\end{table*}

%% file: content/conclusion.tex
\section{Conclusion}
Multi-vector multimodal embeddings have been widely adopted for visual document retrieval, but their patch-level representations incur substantial index storage and search cost. In this paper, we propose \method, which generates a short sequence of multi-vector embeddings auto-regressively. Across the evaluated backbones, \method models are competitive with post-hoc and trainable compression baselines, and CausalPali surpasses its full-vector ColPali counterpart. Controlled experiments further show useful prefix representations and an empirical budget--quality scaling trend with high-budget saturation. The system results also delimit the practical trade-off: compact representations reduce storage and late-interaction time, whereas auto-regressive query encoding increases online latency in the current implementation. These findings suggest that generative embeddings offer a promising alternative for storage-constrained visual document retrieval, warranting further investigation.

\section*{Acknowledgements}
This work was supported by the Alibaba Innovative Research (AIR) Program (Grant No.~64575662); Alibaba Research Intern Program; National Natural Science Foundation of China (Grant No.~62506318); Guangdong Provincial Department of Education Project (Grant No.~2024KQNCX028); Scientific Research Projects for the Higher-educational Institutions (Grant No.~2024312096), Education Bureau of Guangzhou Municipality; Guangzhou-HKUST(GZ) Joint Funding Program (Grant No.~2025A03J3957), Education Bureau of Guangzhou Municipality.

%% file: content/appendix.tex
\section{Broader Impact Statement}  
\paragraph{Ethical Considerations.}  
\method is trained and evaluated on publicly available benchmark data, but deployment inherits the limitations of its backbone and training corpus, including possible demographic, linguistic, or domain biases. Retrieval errors may be consequential in settings such as healthcare or legal search; the system should therefore support provenance inspection, access control, and human verification rather than be treated as an authoritative decision maker. Compact indexes can also make large-scale document search easier to deploy, which increases the importance of respecting privacy, copyright, and data-retention requirements.
\paragraph{Societal Implications.}  
\method reduces document-vector count and can reduce physical index storage when the embedding dimension is controlled. This may broaden access to multimodal document retrieval in storage-constrained environments. However, the present implementation trades these benefits for additional query-encoding latency and does not eliminate the compute or environmental cost of the underlying MLLM. The storage, accuracy, latency, and governance requirements of a target application should therefore be evaluated jointly.
\section{Use of LLMs}
We used AI assistants for two purposes: (1) generating routine code and boilerplate functions, which were subsequently reviewed and debugged by humans, and (2) performing grammatical review and sentence-level editing of the manuscript. The research methodology, findings, and analysis were independently proposed and conducted.
\section{Additional Analysis and Proofs}\label{app:analysis}
\subsection{Analysis of Training Efficiency}\label{app:comp}
\begin{theorem}[Prefix-dependency count under uniform selection]
\label{thm:training-efficiency}
Under a MaxSim-based late-interaction objective, suppose that the maximally matched document-output index for each query vector is uniformly distributed over its valid range. Let $L^{\mathrm{forward}}$ and $L^{\mathrm{causal}}$ count the number of preceding document-context positions connected to the selected outputs under the forward-based and auto-regressive constructions, respectively. Then,
\begin{equation}
\begin{aligned}
\mathbb{E}[L^{\mathrm{forward}}] &= \frac{N_t (N_v-1)}{2}, \\
\mathbb{E}[L^{\mathrm{causal}}] &= N_q\!\left( N_v + \frac{N_d-1}{2} \right).
\end{aligned}
\end{equation}
When $N_q=N_t$ and $N_d\ll N_v$, their ratio is
\begin{equation}
\frac{\mathbb{E}[L^{\mathrm{causal}}]}{\mathbb{E}[L^{\mathrm{forward}}]}
=\frac{2N_v+N_d-1}{N_v-1}\approx 2.
\end{equation}
\end{theorem}
\begin{proof}
We count document-side prefix dependencies associated with the output selected by MaxSim. This is a structural count, rather than a claim about gradient magnitude or optimization speed.

\paragraph{Forward-based Multi-vector Embedding.}
In the forward-based paradigm, document embeddings are indexed as
$\{d_1, d_2, \dots, d_{N_v}\}$.
For each query token, the MaxSim operation selects a document index
\[
j \sim \mathrm{Uniform}\{1,\dots,N_v\}.
\]
The number of document-context positions preceding $d_j$ is $j-1$.
Thus, the expected preceding-token coverage for a single query token is
\begin{equation}
\mathbb{E}[L^{\mathrm{forward}}_i]
= \frac{1}{N_v}\sum_{j=1}^{N_v} (j-1)
= \frac{N_v-1}{2}.
\end{equation}
Summing over all $N_t$ query tokens yields
\begin{equation}
\mathbb{E}[L^{\mathrm{forward}}]
= N_t \cdot \frac{N_v-1}{2}.
\end{equation}

\paragraph{Auto-regressive Embedding (\method).}
In the auto-regressive paradigm, document embeddings are generated sequentially after the visual tokens,
with indices
\[
\{d_{N_v+1}, \dots, d_{N_v+N_d}\}.
\]
For each query token, the MaxSim-selected index satisfies
\[
j \sim \mathrm{Uniform}\{N_v+1, \dots, N_v+N_d\}.
\]
Due to the causal dependency, each selected generated output is connected to the $N_v$ input-context positions and its preceding generated states. Hence, a selected generation step $k$ has $N_v+k-1$ preceding positions.

The expected preceding-token coverage for a single query token is therefore
\begin{equation}\begin{aligned}
\mathbb{E}[L^{\mathrm{causal}}_i]
&= \frac{1}{N_d}\sum_{k=1}^{N_d} (N_v + k - 1) \\
&= N_v + \frac{N_d - 1}{2}.
\end{aligned}\end{equation}
Aggregating over all $N_q$ query tokens gives
\begin{equation}
\mathbb{E}[L^{\mathrm{causal}}]
= N_q\left( N_v + \frac{N_d-1}{2} \right).
\end{equation}

\paragraph{Comparison.}
Dividing the two expectations with $N_q=N_t$ gives the stated ratio, which approaches two when $N_d\ll N_v$. This completes the proof.
\end{proof}
\paragraph{Scope of the result.}
The theorem is a conditional bookkeeping result. It neither establishes an order-of-magnitude advantage nor implies lower training latency. Moreover, Figure~\ref{fig:dist_validation} shows that uniform selection is a useful synthetic reference but is rejected on the measured real data. We therefore use the theorem only to illustrate the additional causal dependency paths and rely on the controlled experiment in Appendix~\ref{app:controlled_prefix} for empirical evidence.
\input{tables/alg_train}
\subsection{Analysis of Proposed Training Objective}
In this section, we further analyze the gradient flow of \method under the overall objective in Eq.~\ref{eq:loss}.
All symbols follow the main body. For convenience, we define
\newcommand{\Score}{\mathcal{S}}
\newcommand{\ip}[2]{\langle #1,#2\rangle}
and recall the late-interaction (MaxSim) scoring function
\begin{equation}\begin{aligned}
\Score(Q,D)
&= \sum_{i=1}^{N_q} \max_{1\le j\le N_d} \ip{q_i}{d_j}
= \sum_{i=1}^{N_q} q_i^\top d_{j^*(i;Q,D)}, \\
j^*(i;Q,D) &\triangleq \arg\max_{1\le j\le N_d} \ip{q_i}{d_j}.
\end{aligned}\end{equation}
\textbf{Prefix-recursive gradient decomposition.}
Due to the auto-regressive dependency, both query- and document-side embeddings admit a prefix-recursive total derivative:
\begin{equation}\label{eq:ddt_prefix}\begin{aligned}
\frac{d q_i}{d\theta}
&= \frac{\partial q_i}{\partial\theta}
  + \sum_{\tau<i}\frac{\partial q_i}{\partial q_\tau}\frac{d q_\tau}{d\theta},
\qquad i=1,\dots,N_q, \\
\frac{d d_t}{d\theta}
&= \frac{\partial d_t}{\partial\theta}
  + \sum_{k<t}\frac{\partial d_t}{\partial d_k}\frac{d d_k}{d\theta},
\qquad t=1,\dots,N_d.
\end{aligned}\end{equation}
Therefore, the score gradient can be written as
\begin{equation}\label{eq:score_grad_prefix}\begin{aligned}
\nabla_\theta \Score(Q,D)
&= \sum_{i=1}^{N_q}
\Bigg[
\underbrace{\Big(\frac{d q_i}{d\theta}\Big)^\top d_{j^*(i;Q,D)}}_{\textbf{query-side signal}}
+
\underbrace{q_i^\top \Big(\frac{d d_{j^*(i;Q,D)}}{d\theta}\Big)}_{\textbf{document-side signal}}
\Bigg]\\
&= \sum_{i=1}^{N_q}
\Bigg[
\Big(\frac{\partial q_i}{\partial\theta}
  + \sum_{\tau<i}\frac{\partial q_i}{\partial q_\tau}\frac{d q_\tau}{d\theta}\Big)^\top
  d_{j^*(i;Q,D)}
+
q_i^\top
\Big(\frac{\partial d_{j^*(i;Q,D)}}{\partial\theta}
  + \sum_{k<j^*(i;Q,D)}\frac{\partial d_{j^*(i;Q,D)}}{\partial d_k}\frac{d d_k}{d\theta}\Big)
\Bigg].
\end{aligned}\end{equation}
Eq.~\ref{eq:score_grad_prefix} explicitly shows that gradients can penetrate the visual context and propagate through the entire auto-regressive generation chain on both sides.
\par\textbf{Contrastive alignment loss \texorpdfstring{$\mathcal{L}_m$}{Lm}.}
Recall the contrastive objective
\begin{equation}\begin{aligned}
\mathcal{L}_m
&= \frac{1}{b}\sum_{k=1}^b \log\Big(1+\exp(s_k^- - s_k^+)\Big), \\
s_k^+ &\triangleq \Score(Q_k,D_k),\qquad
s_k^- \triangleq \Score(Q_k,D_{\ell^*(k)}),\quad
\ell^*(k)=\arg\max_{\ell\ne k}\Score(Q_k,D_\ell), \\
\delta_k &\triangleq \sigma(s_k^- - s_k^+) - 1,
\end{aligned}\end{equation}
whose gradient is
\begin{equation}\begin{aligned}
\nabla_\theta \mathcal{L}_m
&= \frac{1}{b}\sum_{k=1}^b \delta_k
\Big( \nabla_\theta s_k^+ - \nabla_\theta s_k^- \Big) \\
&= \frac{1}{b}\sum_{k=1}^b \delta_k
\Big( \nabla_\theta \Score(Q_k,D_k) - \nabla_\theta \Score(Q_k,D_{\ell^*(k)}) \Big),
\end{aligned}\end{equation}
where each $\nabla_\theta \Score(\cdot,\cdot)$ is expanded by Eq.~\ref{eq:score_grad_prefix}.
\par\textbf{Progressive refinement loss \texorpdfstring{$\mathcal{L}_d$}{Ld}.}
For the progressive refinement objective, we have
\begin{equation}\begin{aligned}
\mathcal{L}_d
&= -\Big(\Score(q, D^+_{1:N_d})-\Score(q, d_1^+)\Big)
   +\Big(\Score(q, D^-_{1:N_d})-\Score(q, d_1^-)\Big),
\end{aligned}\end{equation}
and thus
\begin{equation}\begin{aligned}
\nabla_\theta \mathcal{L}_d
&= -\nabla_\theta \Score(q, D^+_{1:N_d})
   +\nabla_\theta \Score(q, d_1^+)
   +\nabla_\theta \Score(q, D^-_{1:N_d})
   -\nabla_\theta \Score(q, d_1^-),
\end{aligned}\end{equation}
where the document-side prefix decomposition is governed by Eq.~\ref{eq:ddt_prefix}.
\par\textbf{Diversity regularization loss \texorpdfstring{$\mathcal{L}_q$}{Lq}.}
Similarly, the diversity regularizer is defined as
\begin{equation}\begin{aligned}
\mathcal{L}_q
&= \frac{1}{Z}\sum_{i=1}^{N_q}\sum_{j\ne i}
\Big(c_{ij}\Big)^2,
\qquad
c_{ij}\triangleq \frac{q_i^\top q_j}{\|q_i\|\|q_j\|},
\end{aligned}\end{equation}
with the per-token gradient
\begin{equation}\begin{aligned}
g_i \;\triangleq\; \frac{\partial \mathcal{L}_q}{\partial q_i}
&= \frac{2}{Z}\sum_{j\ne i}
c_{ij}\left(
\frac{q_j}{\|q_i\|\|q_j\|}
-\frac{c_{ij}}{\|q_i\|^2}q_i
\right). \label{eq:gi_diversity}
\end{aligned}\end{equation}
Combining with the prefix-recursive derivative in Eq.~\ref{eq:ddt_prefix}, we obtain
\begin{equation}\begin{aligned}
\nabla_\theta \mathcal{L}_q
&= \sum_{i=1}^{N_q} g_i^\top \frac{d q_i}{d\theta}
= \sum_{i=1}^{N_q} g_i^\top
\left(
\frac{\partial q_i}{\partial\theta}
+\sum_{\tau<i}\frac{\partial q_i}{\partial q_\tau}\frac{d q_\tau}{d\theta}
\right), \label{eq:Lq_grad_prefix}
\end{aligned}\end{equation}
which explicitly exposes the query-side prefix recursion.
\par\textbf{Overall gradient.}
Finally, the gradient of the overall objective $\mathcal{L}=\lambda_m\mathcal{L}_m+\lambda_d\mathcal{L}_d+\lambda_q\mathcal{L}_q$ is
\begin{equation}
\begin{aligned}
\nabla_\theta \mathcal{L}
&= \lambda_m \nabla_\theta \mathcal{L}_m
 + \lambda_d \nabla_\theta \mathcal{L}_d
 + \lambda_q \nabla_\theta \mathcal{L}_q \\
&=
\underbrace{\lambda_m\cdot \frac{1}{b}\sum_{k=1}^b \delta_k
\Big( \nabla_\theta \Score(Q_k,D_k) - \nabla_\theta \Score(Q_k,D_{\ell^*(k)}) \Big)}_{\textbf{MaxSim + AR chains}}
\\ &\quad
+\underbrace{\lambda_d\cdot\Big(
-\nabla_\theta \Score(q, D^+_{1:N_d})
+\nabla_\theta \Score(q, d_1^+)
+\nabla_\theta \Score(q, D^-_{1:N_d})
-\nabla_\theta \Score(q, d_1^-)
\Big)}_{\textbf{Late-early Separation}}
\\ &\quad
+\underbrace{\lambda_q\cdot \sum_{i=1}^{N_q} g_i^\top
\left(
\frac{\partial q_i}{\partial\theta}
+\sum_{\tau<i}\frac{\partial q_i}{\partial q_\tau}\frac{d q_\tau}{d\theta}
\right)}_{\textbf{Query Orthogonalization}}.
\end{aligned}
\end{equation}
\subsection{Training Process of \method}
To describe the training process of our framework more precisely, we illustrate it in pseudocode, as shown in Algorithm~\ref{alg:train}.
\section{More Related Work}
\subsection{Matryoshka Representation Learning}\label{app:matryoshka_representation_learning}
Matryoshka Representation Learning (MRL), proposed by~\citep{kusupati2022matryoshka}, aims to encode features at multiple granularities within a single vector using a nested structure. By incorporating training-time objectives across various truncation dimensions, MRL enables a flexible trade-off between performance and efficiency. This method has been widely adopted in recent embedding models~\citep{meng2025vlm2vec2, qwen3vlembedding, zhuang2024starbucksv2}.
Recently, MetaEmbed~\citep{xiao2025metaembed} extended this idea to a token-wise Matryoshka structure by truncating prefix sequences of varying lengths during multi-vector training, where the selectable lengths are predefined.
In our work, we observe the natural emergence of the Matryoshka phenomenon in \method models as a consequence of autoregressive training. This observation offers new insights into the training dynamics and intrinsic characteristics of our proposed method.
\section{Training Details}
\label{app:train}
We trained CausalQwen and CausalPali on the \textsc{ColPali} training set using four NVIDIA A800 GPUs for one epoch, with a batch size of 8 per GPU and a global batch size of 32. During training, we disabled gradient checkpointing and enabled KV caching. Low-Rank Adaptation (LoRA) was employed with rank $r=32$ and $\alpha=32$ on the language-model linear projections; the original backbone parameters, including the vision encoder and token-embedding layer, remained frozen. The learning rate was set to $2\times 10^{-5}$, and the temperature was maintained at 0.02. For the \method loss, the hyperparameters were configured as $\lambda_m=1$ and $\lambda_d=\lambda_q=0.1$.
Our training framework is based on the official implementation of \textsc{ColPali}\footnote{\url{https://github.com/illuin-tech/colpali}}. Generated hidden states are appended directly through the model's \texttt{inputs\_embeds} interface, receive the next position identifiers from the backbone's native positional scheme, and bypass the discrete token-embedding lookup. The training cache is not detached, preserving the computation graph across auto-regressive steps. In our main experiments, we set $N_q=16$ and $N_d=32$; CausalQwen and CausalPali therefore emit $16\times 2{,}048$ query and $32\times 2{,}048$ document representations. We also train models with $N_d\in\{64,128\}$ for the scaling study and a lightweight 128-dimensional projection for the storage study.
\section{Supplemental Results}\label{app:exp}

\subsection{MaxSim Selection-distribution Diagnostic}
\label{app:uniform_dist}
The conditional calculation in Section~\ref{app:comp} uses uniform MaxSim selection as a reference case. We test that reference on synthetic samples drawn to be uniform and on observed selections from the retrieval models.

\begin{figure}[h]
    \centering
    \begin{minipage}{0.48\linewidth}
        \centering
        \includegraphics[width=\linewidth]{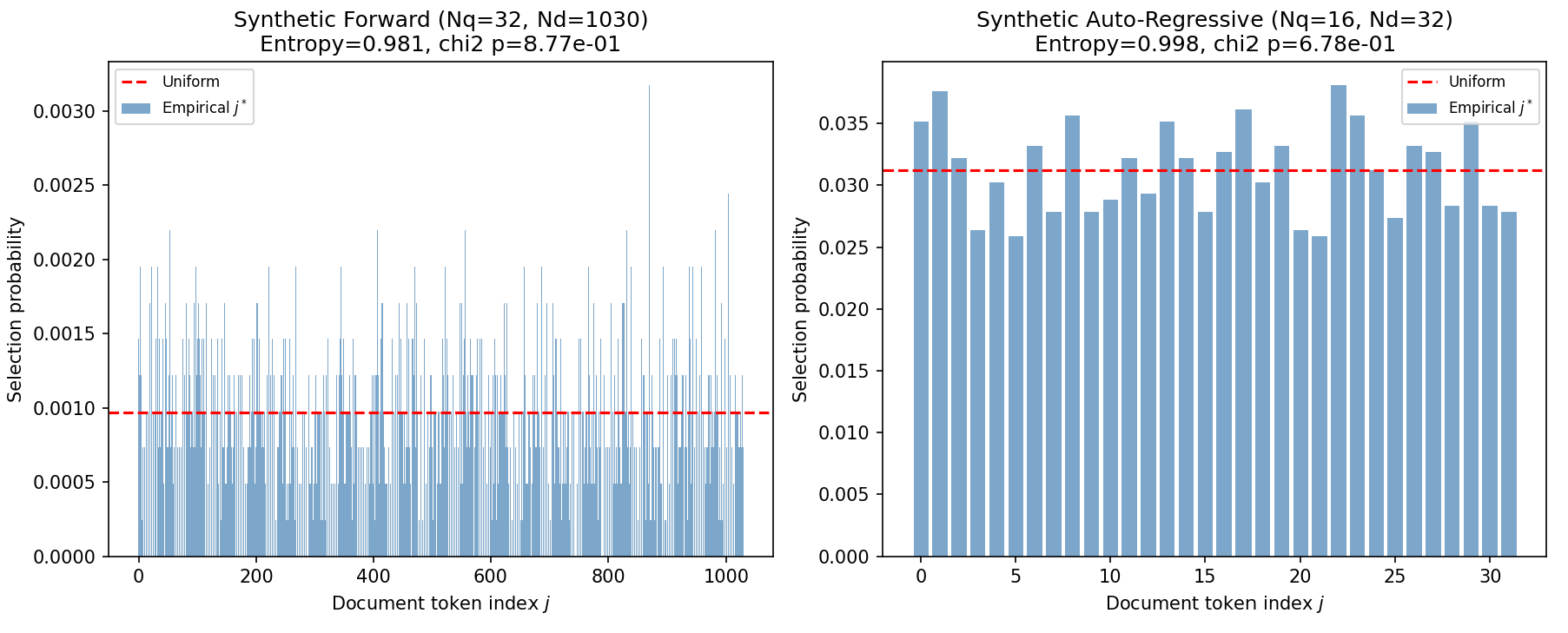}
    \end{minipage}
    \hfill
    \begin{minipage}{0.48\linewidth}
        \centering
        \includegraphics[width=\linewidth]{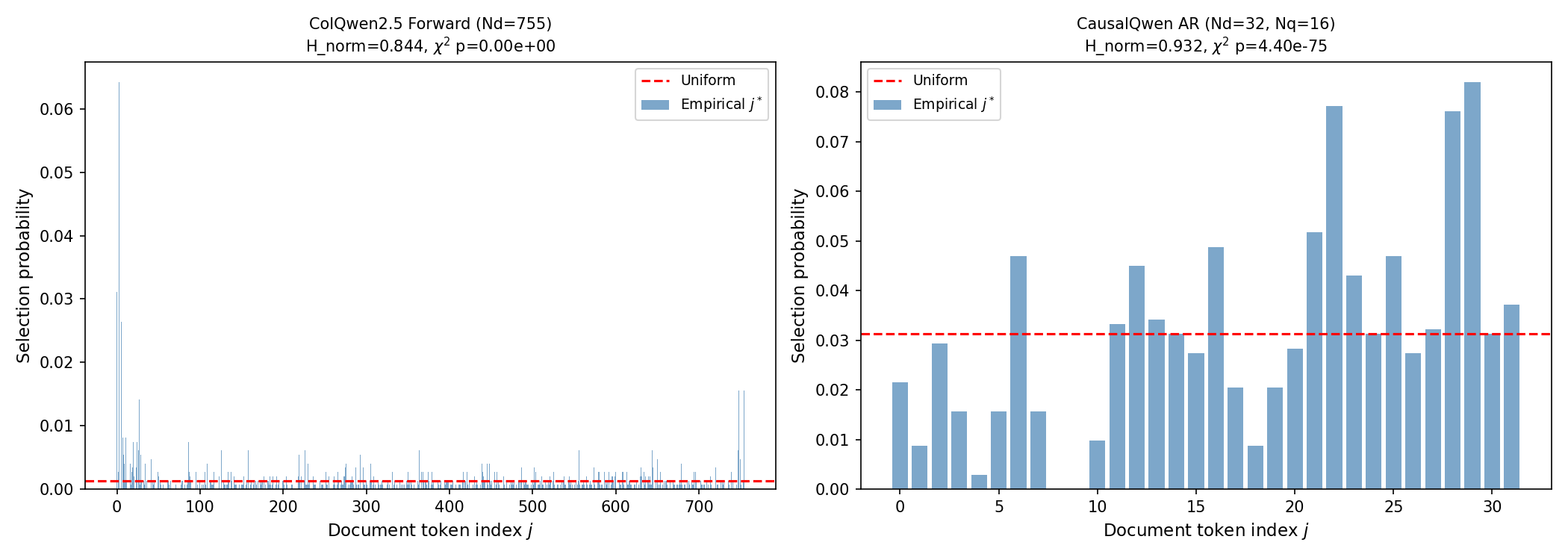}
    \end{minipage}
    \caption{MaxSim-selected document indices for the synthetic uniform reference (left) and observed model selections (right). The synthetic $\chi^2$ tests do not reject uniformity ($p=0.877$ and $p=0.678$), whereas the real-data tests reject it for both the forward and auto-regressive models ($p<10^{-74}$).}
    \label{fig:dist_validation}
\end{figure}

As shown in Figure~\ref{fig:dist_validation}, the synthetic sanity check behaves as expected, but the observed selections are non-uniform. Their normalized entropies remain relatively high (0.844 for ColQwen2.5 and 0.932 for CausalQwen), yet the concentration at particular indices is statistically incompatible with exact uniformity. Accordingly, Theorem~\ref{thm:training-efficiency} should be read as a stylized structural comparison rather than an empirical estimate of gradient coverage.

\subsection{Training Dynamics}\label{app:training_efficiency}
We visualize the training dynamics of CausalQwen and CausalPali in Figure~\ref{fig:loss}. As shown in Figure~\ref{fig:loss}, both training and evaluation losses decrease and begin to plateau within the configured single epoch of approximately 10{,}000 steps. This observation supports the sufficiency of the one-epoch budget used in our experiments, but does not establish that one epoch is universally optimal or that \method converges faster than separately tuned baselines.

\begin{figure}[t]
\centering
\includegraphics[width=0.48\textwidth]{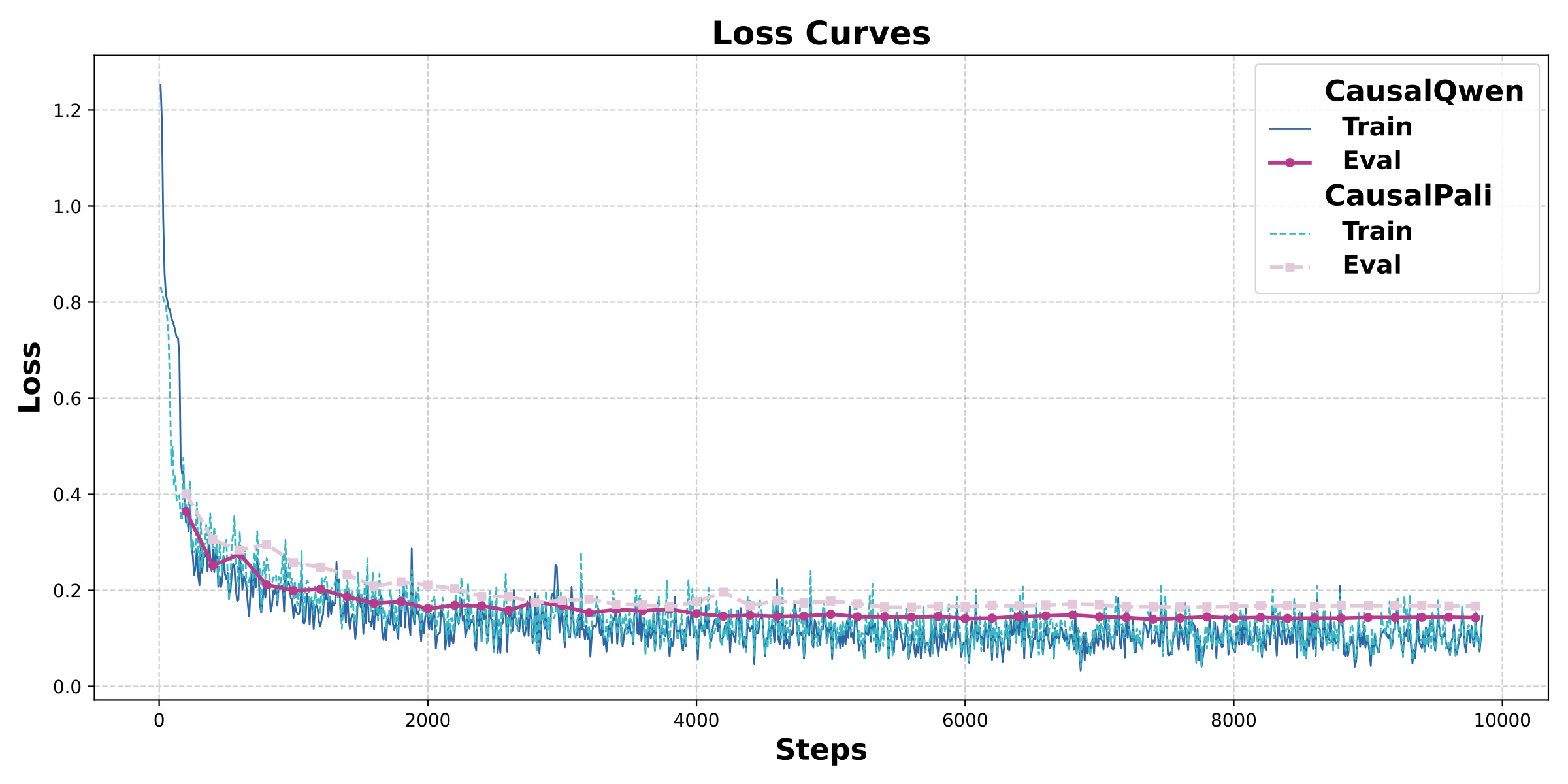}
\caption{Training and evaluation loss curves of CausalQwen and CausalPali over one epoch.}
\label{fig:loss}
\end{figure}

\subsection{Comprehensive Latency Breakdown}
\label{app:latency_breakdown}
To provide a complete picture of the latency trade-offs, we present a detailed end-to-end latency breakdown in Table~\ref{tab:latency_e2e}. The benchmarks are averaged over 5 runs with 512 samples and a batch size of 64. Here, $T_f$ denotes forward pass time and $T_a$ denotes aggregation/autoregressive time.

\begin{table}[h]
\centering
\caption{End-to-end latency breakdown (in milliseconds) for 512 samples with a batch size of 64. Auto-regressive query encoding increases online cost, while 32 document vectors reduce the Late Interaction phase by 11.8$\times$ relative to Vanilla.}
\label{tab:latency_e2e}
\renewcommand{\arraystretch}{1.1}
\begin{tabular}{l|c|c|c|c|c}
\toprule
\textbf{Method} & \textbf{Doc Tokens} & \textbf{Doc Encode} & \textbf{Query Encode} & \textbf{Late Interaction} & \textbf{Total E2E} \\
& & $(T_f + T_a)$ & $(T_f + T_a)$ & & (ms) \\
\midrule
Vanilla    & 759 & 80309.26 & 219.81 & 44.30 & 80573.37 \\
1D-Pool    & 32  & 80442.32 & 219.81 & 5.28  & 80667.42 \\
KMeans     & 32  & 102459.53 & 219.81 & 4.39 & 102683.74 \\
SemCluster & 32  & 97120.73 & 219.81 & 4.04  & 97344.58 \\
\rowcolor{black!10} Ours & 32 & 81215.20 & 702.68 & 3.74 & 81921.61 \\
\bottomrule
\end{tabular}
\end{table}

In this 512-sample measurement, \method is slower end-to-end than both Vanilla (81,921.61 vs. 80,573.37 ms) and 1D-Pool (80,667.42 ms); the MaxSim reduction does not compensate for auto-regressive encoding. The result therefore does not support a universal end-to-end speedup. The following corpus-scale profiling separates the one-time document indexing path, online query encoding, and search so that the deployment trade-off is explicit.

\begin{table}[h]
\centering
\caption{Controlled 10K-representation storage and MaxSim profile. Representation size assumes FP16 values. Relative storage is normalized to ColQwen.}
\label{tab:storage_10k}
\renewcommand{\arraystretch}{1.1}
\resizebox{\linewidth}{!}{
\begin{tabular}{l|c|c|c|c}
\toprule
\textbf{Method} & \textbf{Representation} & \textbf{MaxSim p50 (ms)} & \textbf{Size (GB)} & \textbf{Relative Storage} \\
\midrule
ColQwen & $4{,}962\times128$ & 8.340 & 12.703 & 1.0000$\times$ \\
CausalQwen & $32\times2{,}048$ & 0.550 & 1.311 & 0.1030$\times$ \\
CausalQwen-128D & $32\times128$ & 0.120 & 0.082 & 0.0065$\times$ \\
\bottomrule
\end{tabular}}
\end{table}

Table~\ref{tab:storage_10k} demonstrates why vector count and physical size should be reported separately. The unprojected CausalQwen representation contains 155.1$\times$ fewer vectors than ColQwen, but each vector is 16$\times$ wider, resulting in a 9.7$\times$ byte reduction. With the 128-dimensional projection, both representations have the same width and the byte reduction matches the 155.1$\times$ vector-count reduction. In this controlled MaxSim benchmark, the projected representation is 69.5$\times$ faster than ColQwen.

\begin{table}[h]
\centering
\caption{Measured online retrieval path with a 100K corpus. Document encoding is an offline indexing cost and is excluded from encode-plus-search latency.}
\label{tab:retrieval_100k}
\renewcommand{\arraystretch}{1.1}
\resizebox{\linewidth}{!}{
\begin{tabular}{l|c|c|c|c}
\toprule
\textbf{Method} & \textbf{Index (GB)} & \textbf{Query Encode (ms)} & \textbf{Search (ms)} & \textbf{Encode + Search (ms)} \\
\midrule
ColQwen & 19.719 & 21.610 & 15.947 & 37.557 \\
CausalQwen-2048D & 13.107 & 364.591 & 5.173 & 369.764 \\
CausalQwen-128D & \textbf{0.819} & 334.975 & \textbf{0.885} & 335.860 \\
\bottomrule
\end{tabular}}
\end{table}

At 100K documents, compact representations reduce both the measured index and search time. Nevertheless, query encoding dominates the online path for both CausalQwen variants, making their total online latency substantially higher than ColQwen. The favorable operating point therefore depends on whether the application prioritizes index memory and search throughput or per-query response time.

\begin{table}[h]
\centering
\caption{Effect of the learned 128-dimensional projection on six representative subsets (nDCG@5). Both variants use 32 document vectors.}
\label{tab:projection_128d}
\renewcommand{\arraystretch}{1.1}
\resizebox{\linewidth}{!}{
\begin{tabular}{l|cccccc|c}
\toprule
\textbf{Dimension} & \textbf{DocVQA} & \textbf{TAT-DQA} & \textbf{Bio-EN} & \textbf{Bio-Multi} & \textbf{Energy} & \textbf{Finance-FR} & \textbf{Mean} \\
\midrule
2,048D & 0.55387 & 0.61765 & 0.55606 & 0.50709 & 0.45691 & 0.29175 & 0.49722 \\
128D & 0.53715 & 0.59398 & 0.53315 & 0.46716 & 0.41062 & 0.25917 & 0.46687 \\
\bottomrule
\end{tabular}}
\end{table}

The 128-dimensional projection is trained on top of a frozen CausalQwen and retains 93.90\% of the six-subset mean nDCG@5 while reducing the per-page representation from 131,072 to 8,192 bytes in FP16. This provides an optional storage--quality operating point rather than changing the dimensionality of the main model.

\subsection{Test-time Scaling}
In Figure~\ref{fig:app_scaling}, we present additional results on the test-time scaling behavior of \method across different datasets. The experiments are primarily conducted using CausalQwen. The overall trends consistently align with our analysis in Section~\ref{sec:tts}.

\begin{figure}[t]
    \centering
    \includegraphics[width=\linewidth]{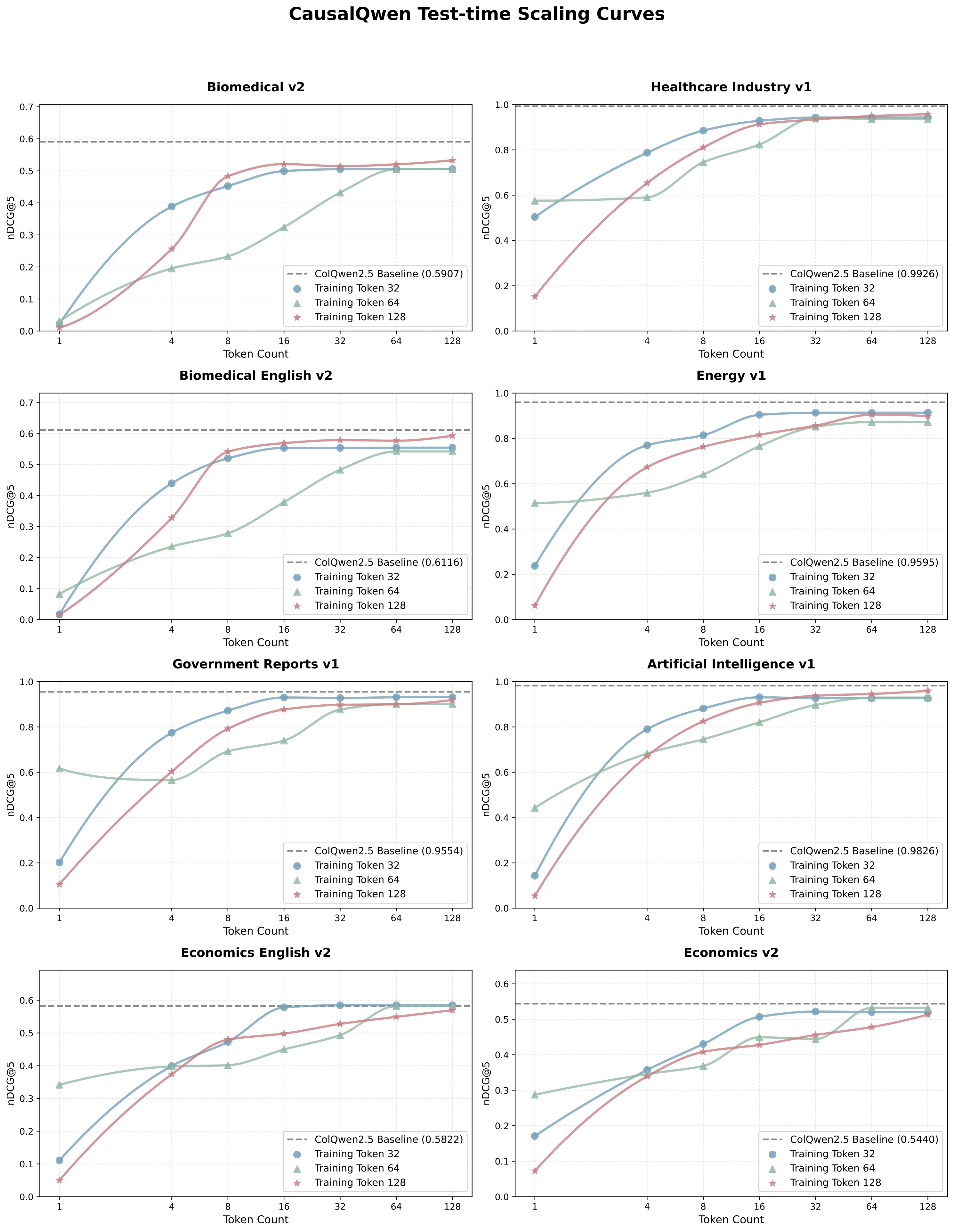}
    \caption{Additional results illustrating the test-time scaling characteristics of CausalQwen on \vidore~V1 and V2.}
    \label{fig:app_scaling}
\end{figure}
\FloatBarrier

\subsection{Trainable Compression Baselines}
\label{app:learned_baselines}
We reproduce MetaEmbed with the Qwen2.5-VL-3B model family, the ColPali training set, LoRA rank 32, and a global batch size of 32. Table~\ref{tab:metaembed} compares different query/document budgets on the same 25-subset evaluation used in this reproduction.

\begin{table}[h]
\centering
\caption{Comparison with the trainable MetaEmbed reproduction. The budget is reported as query/document vectors, and the metric is mean nDCG@5 over 25 evaluated subsets.}
\label{tab:metaembed}
\renewcommand{\arraystretch}{1.1}
\begin{tabular}{l|c|c|c}
\toprule
\textbf{Method} & \textbf{Query/Doc. Budget} & \textbf{Doc. Vectors} & \textbf{Mean nDCG@5} \\
\midrule
MetaEmbed & 1/1 & 1 & 0.48341 \\
MetaEmbed & 2/4 & 4 & \underline{0.51755} \\
MetaEmbed & 4/8 & 8 & 0.51428 \\
MetaEmbed & 8/16 & 16 & 0.50593 \\
MetaEmbed & 16/32 & 32 & 0.49456 \\
\rowcolor{black!10} CausalQwen & 16/32 & 32 & \textbf{0.60915} \\
\bottomrule
\end{tabular}
\end{table}

CausalQwen exceeds the best reproduced MetaEmbed operating point by 0.09160 mean nDCG@5. Notably, the best MetaEmbed result occurs at four rather than 32 document vectors, whereas CausalQwen uses the same 16/32 query/document budget as the largest MetaEmbed configuration.

\begin{table}[h]
\centering
\caption{Comparison with the DocPruner reproduction on six representative subsets (nDCG@5). CausalQwen uses 32 document vectors.}
\label{tab:docpruner}
\renewcommand{\arraystretch}{1.1}
\resizebox{\linewidth}{!}{
\begin{tabular}{l|ccc|c}
\toprule
\textbf{Dataset} & \textbf{DocPruner-32} & \textbf{DocPruner-64} & \textbf{DocPruner-128} & \textbf{CausalQwen-32} \\
\midrule
DocVQA & 0.30105 & 0.46517 & 0.54525 & \textbf{0.55387} \\
TAT-DQA & 0.34904 & 0.49340 & \textbf{0.62245} & 0.61765 \\
Biomedical EN & 0.40029 & 0.46615 & 0.54735 & \textbf{0.55606} \\
Biomedical Multilingual & 0.39092 & 0.44947 & \textbf{0.52850} & 0.50709 \\
Energy & 0.25460 & 0.35871 & 0.44165 & \textbf{0.45691} \\
Finance-FR & 0.08862 & 0.15378 & 0.21906 & \textbf{0.29175} \\
\midrule
Mean & 0.29742 & 0.39778 & 0.48404 & \textbf{0.49722} \\
\bottomrule
\end{tabular}}
\end{table}

At 32 document vectors, CausalQwen exceeds DocPruner by 0.19980 mean nDCG@5 on these subsets. DocPruner approaches the same regime at 128 vectors, where its mean remains 0.01318 below CausalQwen-32. These reproduced comparisons complement, rather than replace, the full benchmark tables because their evaluation suites differ.
\FloatBarrier

\subsection{Document-prefix Scaling on Two Backbones}
\label{app:prefix_scaling}
Tables~\ref{tab:qwen_prefix_scaling} and~\ref{tab:pali_prefix_scaling} report the complete prefix-length sweep. The aggregate trend is increasing on both backbones, although Energy and Finance-FR contain small local fluctuations near saturation for CausalQwen.

\begin{table}[h]
\centering
\caption{CausalQwen nDCG@5 at different inference-time document-prefix lengths.}
\label{tab:qwen_prefix_scaling}
\renewcommand{\arraystretch}{1.05}
\resizebox{\linewidth}{!}{
\begin{tabular}{l|ccccccc}
\toprule
\textbf{Dataset} & $N_d=1$ & $N_d=2$ & $N_d=4$ & $N_d=8$ & $N_d=16$ & $N_d=24$ & $N_d=32$ \\
\midrule
DocVQA & 0.09146 & 0.29196 & 0.33487 & 0.47953 & 0.53937 & 0.55258 & 0.55387 \\
TAT-DQA & 0.09899 & 0.32149 & 0.41306 & 0.52308 & 0.60603 & 0.61747 & 0.61765 \\
Biomedical EN & 0.01404 & 0.41172 & 0.44256 & 0.51187 & 0.54414 & 0.55430 & 0.55606 \\
Biomedical Multilingual & 0.02120 & 0.36163 & 0.39322 & 0.44753 & 0.49457 & 0.50685 & 0.50709 \\
Energy & 0.05152 & 0.24740 & 0.29840 & 0.38699 & \textbf{0.45683} & 0.45584 & 0.45691 \\
Finance-FR & 0.02548 & 0.17033 & 0.18432 & 0.26575 & 0.28879 & \textbf{0.29183} & 0.29175 \\
\midrule
Mean & 0.05045 & 0.30075 & 0.34441 & 0.43579 & 0.48829 & 0.49648 & \textbf{0.49722} \\
\bottomrule
\end{tabular}}
\end{table}

\begin{table}[h]
\centering
\caption{CausalPali nDCG@5 at different inference-time document-prefix lengths.}
\label{tab:pali_prefix_scaling}
\renewcommand{\arraystretch}{1.05}
\resizebox{\linewidth}{!}{
\begin{tabular}{l|ccccccc}
\toprule
\textbf{Dataset} & $N_d=1$ & $N_d=2$ & $N_d=4$ & $N_d=8$ & $N_d=16$ & $N_d=24$ & $N_d=32$ \\
\midrule
DocVQA & 0.07225 & 0.24733 & 0.31592 & 0.40543 & 0.45744 & 0.49189 & 0.49627 \\
TAT-DQA & 0.05188 & 0.26564 & 0.33850 & 0.41091 & 0.46193 & 0.51195 & 0.52219 \\
Biomedical EN & 0.02347 & 0.22420 & 0.41078 & 0.42747 & 0.45867 & 0.48550 & 0.49850 \\
Biomedical Multilingual & 0.01995 & 0.21126 & 0.38354 & 0.41122 & 0.44504 & 0.46674 & 0.48084 \\
Energy & 0.00401 & 0.14698 & 0.27224 & 0.29404 & 0.31295 & 0.33684 & 0.34560 \\
Finance-FR & 0.00219 & 0.09132 & 0.13256 & 0.15234 & 0.15245 & \textbf{0.16945} & 0.16941 \\
\midrule
Mean & 0.02896 & 0.19779 & 0.30892 & 0.35024 & 0.38141 & 0.41040 & \textbf{0.41880} \\
\bottomrule
\end{tabular}}
\end{table}

For CausalQwen, the mean improves by 0.44603 from one to 24 vectors and by only 0.00074 from 24 to 32. CausalPali shows a larger 0.00840 gain over the final interval. These results support budget--quality scaling with backbone-dependent saturation, not exact monotonicity on every subset.
\FloatBarrier

\subsection{Controlled Architecture and Prefix-objective Study}
\label{app:controlled_prefix}
We train four variants from the same retrieval initialization with the same ColPali data, $N_q/N_d=16/32$, global batch size, and one-epoch budget. \textit{Main-only} uses only $\mathcal{L}_m$; \textit{current} uses Eq.~\ref{eq:loss}; and \textit{explicit-prefix} directly supervises intermediate document prefixes. The parallel variants produce all compact vectors in a single forward pass.

\begin{table}[h]
\centering
\caption{Controlled architecture and objective comparison (mean nDCG@5). Prefix-AUC is the normalized area under the nDCG@5--budget curve for $N_d\in\{1,2,4,8,16,24,32\}$.}
\label{tab:controlled_prefix}
\renewcommand{\arraystretch}{1.1}
\begin{tabular}{l|ccc|c}
\toprule
\textbf{Variant} & $N_d=16$ & $N_d=24$ & $N_d=32$ & \textbf{Prefix-AUC} \\
\midrule
AR + main-only & 0.51897 & 0.52134 & 0.52350 & 0.47807 \\
Parallel + main-only & 0.47834 & 0.49950 & 0.51811 & 0.42309 \\
AR + current & 0.43745 & 0.51593 & \textbf{0.52470} & -- \\
Parallel + explicit-prefix & 0.51772 & 0.51801 & 0.51768 & -- \\
\bottomrule
\end{tabular}
\end{table}

With the same main objective, auto-regressive generation improves the 32-vector result by 0.00539 and Prefix-AUC by 0.05498 over parallel generation, isolating an architectural contribution beyond fine-tuning. Explicit prefix supervision also yields strong short prefixes in the parallel model, showing that useful prefix behavior is not exclusive to auto-regressive generation. The current \method objective attains the best full-budget result, whereas its lower 16-vector score indicates that $\mathcal{L}_d$ does not calibrate every intermediate prefix. These controlled variants are single training runs, so small differences should be interpreted descriptively rather than as multi-seed significance estimates.
\FloatBarrier

\subsection{Query Budget and Asymmetric Regularization}
\label{app:regularization_ablation}
We train otherwise identical CausalQwen models with $N_q\in\{8,16,32\}$ and $N_d=32$. Table~\ref{tab:query_budget} supports the default $N_q=16$ on all six evaluated subsets.

\begin{table}[h]
\centering
\caption{Query-vector budget sensitivity with $N_d=32$ (nDCG@5).}
\label{tab:query_budget}
\renewcommand{\arraystretch}{1.05}
\resizebox{\linewidth}{!}{
\begin{tabular}{l|cccccc|c}
\toprule
\textbf{Query Budget} & \textbf{DocVQA} & \textbf{TAT-DQA} & \textbf{Bio-EN} & \textbf{Bio-Multi} & \textbf{Energy} & \textbf{Finance-FR} & \textbf{Mean} \\
\midrule
$N_q=8$ & 0.56570 & 0.66790 & 0.55856 & 0.51805 & 0.45376 & 0.25853 & 0.50375 \\
$N_q=16$ & \textbf{0.57563} & \textbf{0.69523} & \textbf{0.57674} & \textbf{0.53152} & \textbf{0.49679} & \textbf{0.27226} & \textbf{0.52470} \\
$N_q=32$ & 0.54802 & 0.68039 & 0.55162 & 0.50439 & 0.46325 & 0.23478 & 0.49707 \\
\bottomrule
\end{tabular}}
\end{table}

The 16-vector query budget exceeds 8 and 32 vectors by 0.02095 and 0.02763 mean nDCG@5, respectively. More query vectors are therefore not uniformly beneficial under a fixed one-epoch training budget.

\begin{table}[h]
\centering
\caption{Document-side diversity regularization ablation (mean nDCG@5).}
\label{tab:document_diversity}
\renewcommand{\arraystretch}{1.1}
\begin{tabular}{c|cc}
\toprule
\textbf{Doc. Budget} & \textbf{Current Objective} & \textbf{$+$ Document Diversity} \\
\midrule
$N_d=16$ & 0.43745 & \textbf{0.44046} \\
$N_d=24$ & \textbf{0.51593} & 0.48773 \\
$N_d=32$ & \textbf{0.52470} & 0.48800 \\
\bottomrule
\end{tabular}
\end{table}

Adding pairwise diversity to document vectors slightly improves the 16-vector prefix but reduces the 32-vector score by 0.03670. This result motivates the asymmetric design: query-side diversity counters homogenization over a short sequence, while document-side refinement remains task-aligned through $\mathcal{L}_d$.
\FloatBarrier

\subsection{Case Study}
\begin{figure*}[t]
    \centering
    \begin{minipage}{\linewidth}
        \centering
\includegraphics[width=\linewidth]{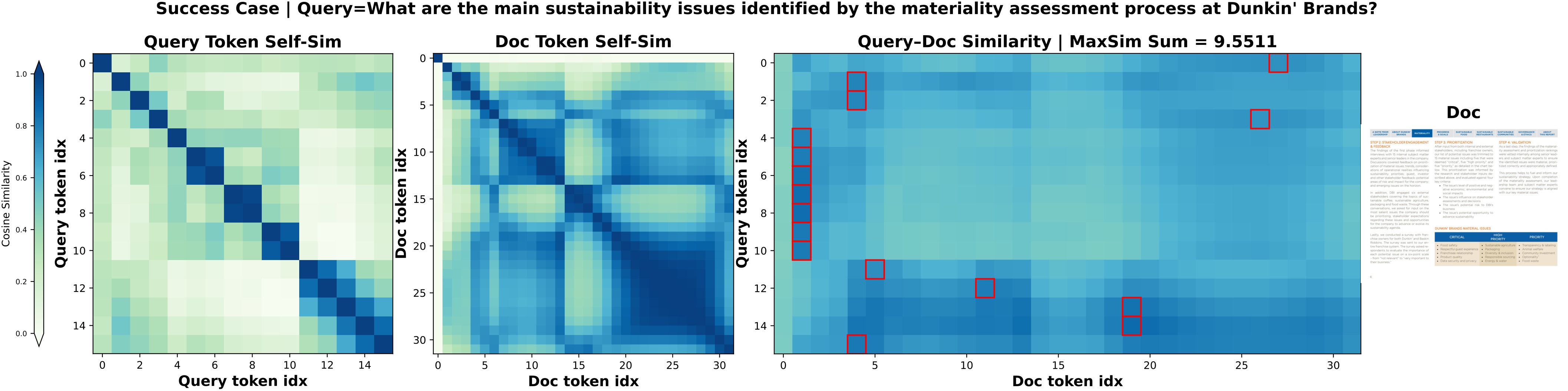}
    \end{minipage}
    \vfill
    \begin{minipage}{\linewidth}
        \centering \includegraphics[width=\linewidth]{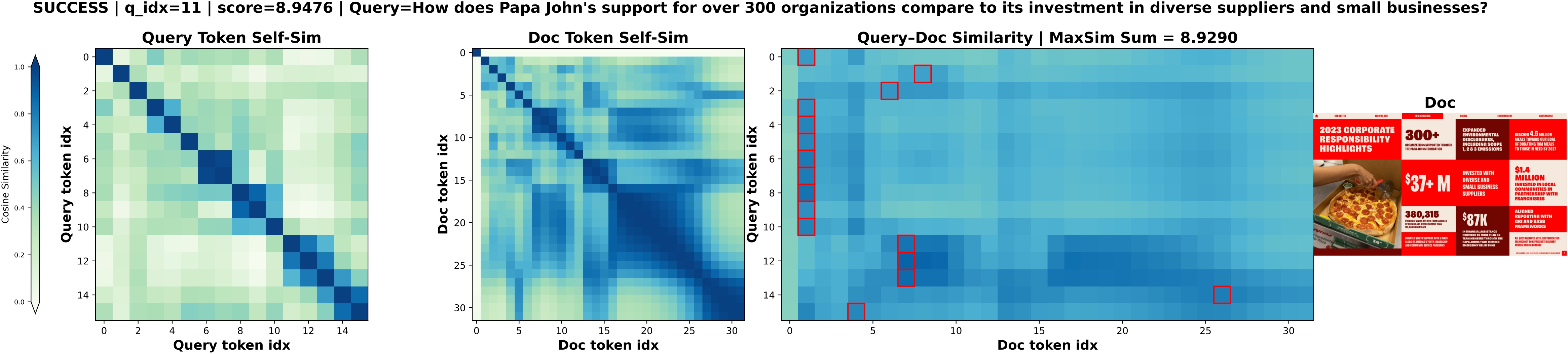}
    \end{minipage}
    \caption{Success cases of CausalQwen.}
    \label{fig:app_qwen_success}
\end{figure*}

\begin{figure*}[t]
    \centering
    \begin{minipage}{\linewidth}
        \centering
\includegraphics[width=\linewidth]{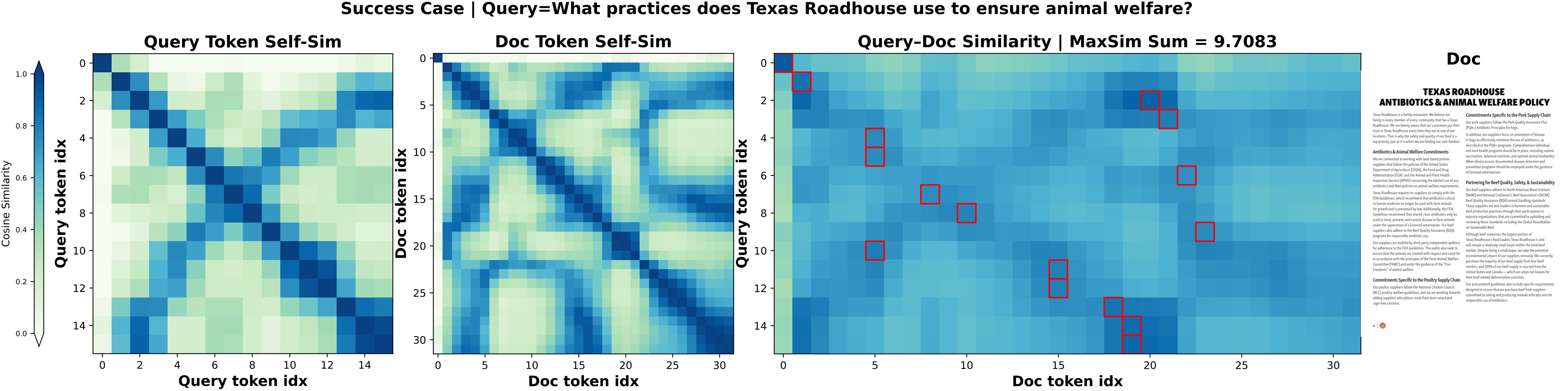}
    \end{minipage}
    \vfill
    \begin{minipage}{\linewidth}
        \centering \includegraphics[width=\linewidth]{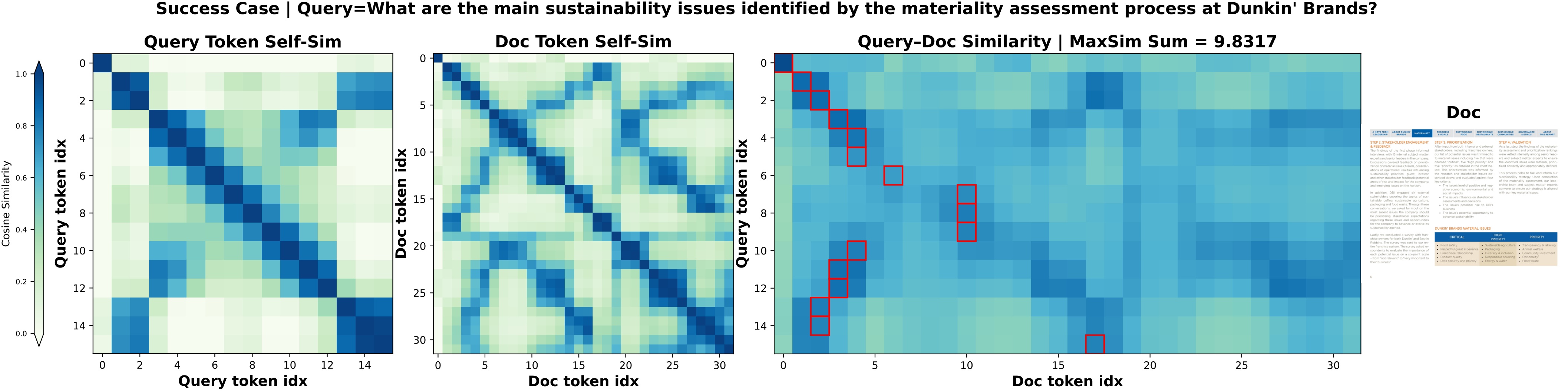}
    \end{minipage}
    \caption{Success cases of CausalPali.}
    \label{fig:app_pali_success}
\end{figure*}

\begin{figure*}[t]
    \centering
    \begin{minipage}{\linewidth}
        \centering
\includegraphics[width=\linewidth]{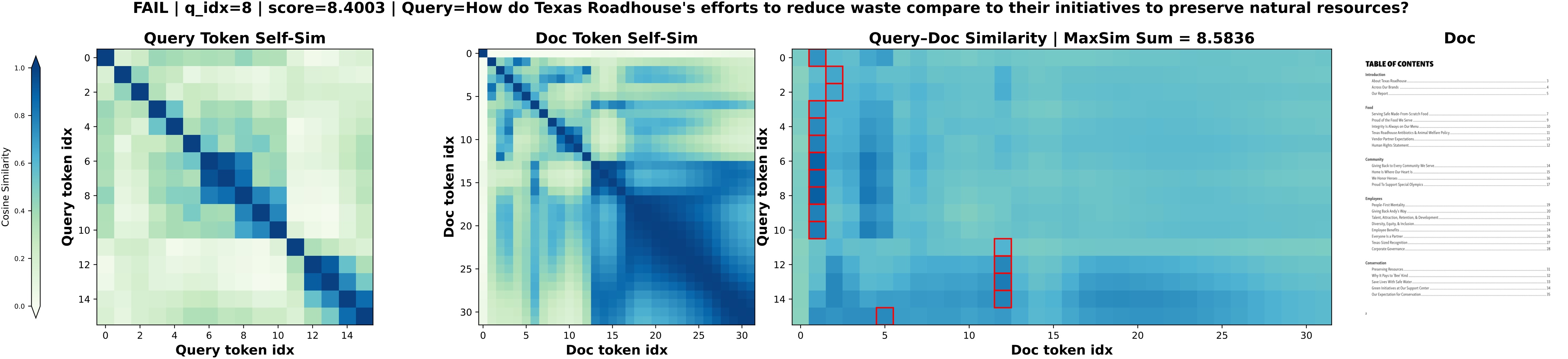}
    \end{minipage}
    \vfill
    \begin{minipage}{\linewidth}
        \centering \includegraphics[width=\linewidth]{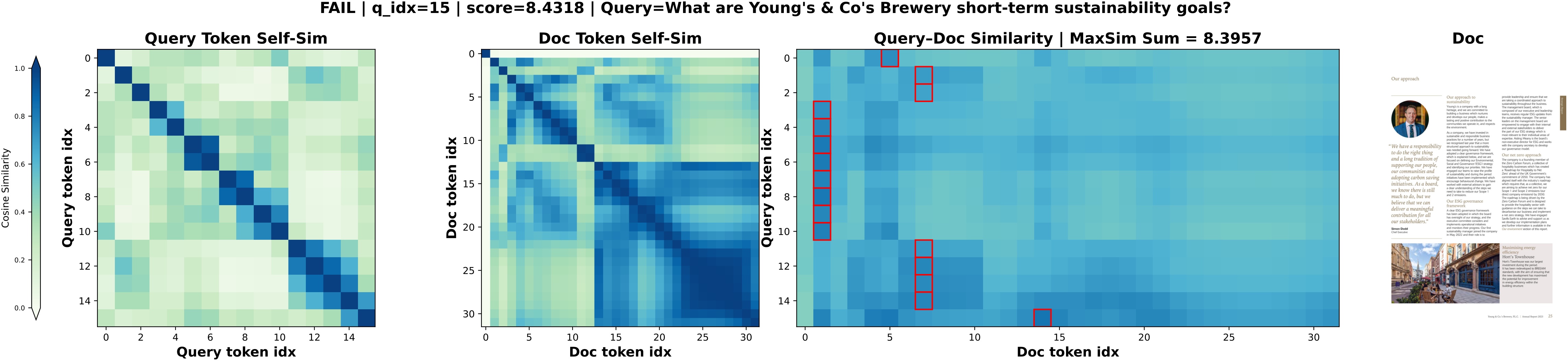}
    \end{minipage}
    \caption{Failure cases of CausalQwen.}
    \label{fig:app_qwen_fail}
\end{figure*}

\begin{figure*}[t]
    \centering
    \begin{minipage}{\linewidth}
        \centering
\includegraphics[width=\linewidth]{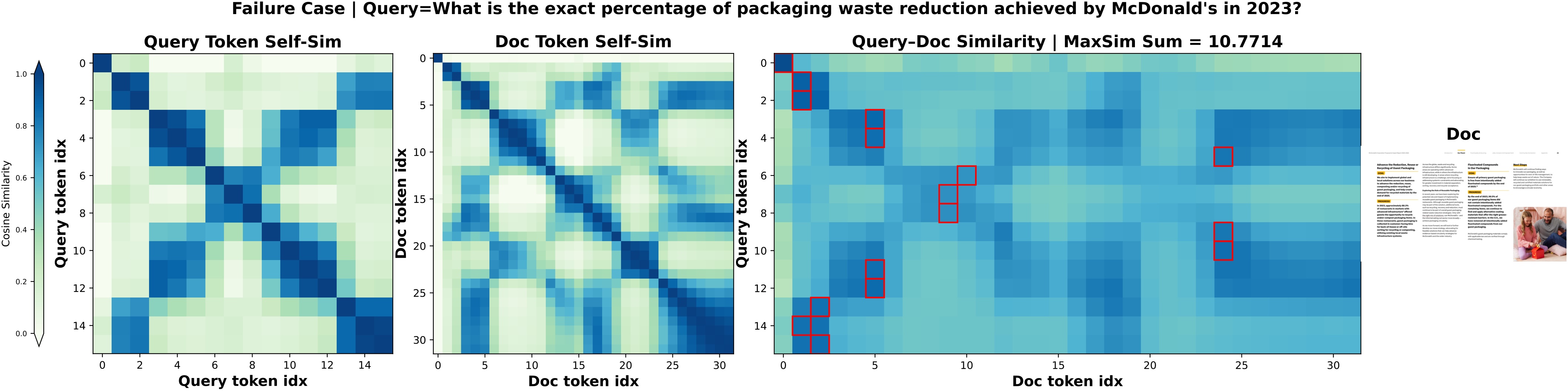}
    \end{minipage}
    \vfill
    \begin{minipage}{\linewidth}
        \centering \includegraphics[width=\linewidth]{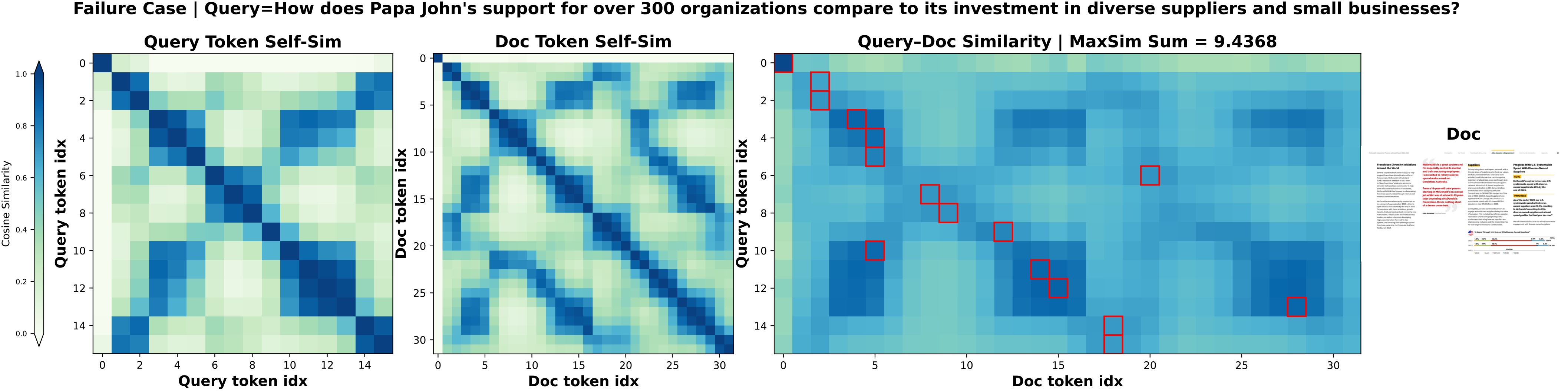}
    \end{minipage}
    \caption{Failure cases of CausalPali.}
    \label{fig:app_pali_fail}
\end{figure*}
In Figures~\ref{fig:app_qwen_success},~\ref{fig:app_pali_success},~\ref{fig:app_qwen_fail}, and~\ref{fig:app_pali_fail}, we present additional examples from CausalQwen and ColPali. These cases illustrate the multi-vector representations generated by the models, along with their corresponding interactions.

%% file: tables/alg_train.tex
\begin{algorithm}[t]
\small
\SetAlgoLined
\LinesNumbered
\SetKwInOut{Input}{Input}
\SetKwInOut{Output}{Output}
\caption{Training Schema of \method with Auto-Regressive Embedding}
\label{alg:train}
\Input{
    Batch of visual documents $I_1, \dots, I_b$ \\
    Batch of text queries $T_1, \dots, T_b$ \\
    Vision-language model $f_\theta = \{\Phi, \Psi\}$ \\
    Number of doc/query tokens $N_d$, $N_q$ \\
    Loss weights $\lambda_m, \lambda_d, \lambda_q$
}
\Output{Total training loss $\mathcal{L}$}

\ForEach{$(I_k, T_k)$ in batch}{
    \tcp{Step 1: Encode document and query input}
    $H^{(v)}_k \gets \Phi(I_k)$ \tcp*{Visual context from image}
    $C_d \gets H^{(v)}_k$, $C_q \gets E(T_k)$
    
    \tcp{Step 2: Generate latent embeddings auto-regressively}
    $D_k = [\,]$; \quad \For{$t=1$ \KwTo $N_d$}{
        $d_t \gets \tau(\Psi([C_d; d_{<t}],\,\texttt{use\_cache=True}))$ \tcp*{Next position ID}
        Append $d_t$ directly to the input-embedding sequence \tcp*{No token lookup}
        $D_k \gets D_k \cup \{d_t\}$
    }
    $Q_k = [\,]$; \quad \For{$t=1$ \KwTo $N_q$}{
        $q_t \gets \tau(\Psi([C_q; q_{<t}],\,\texttt{use\_cache=True}))$ \tcp*{Differentiable cache}
        Append $q_t$ directly to the input-embedding sequence \tcp*{No token lookup}
        $Q_k \gets Q_k \cup \{q_t\}$
    }

    \tcp{Step 3: Compute MaxSim similarity scores}
    $s^+_k \gets \mathcal{S}(Q_k, D_k)$ \tcp*{Positive pair}
    $s^-_k \gets \max_{l\neq k} \mathcal{S}(Q_k, D_l)$ \tcp*{Hard negative}
}

\tcp{Step 4: Compute losses}
$\mathcal{L}_m \gets \frac{1}{b} \sum_{k=1}^{b} \log(1 + \exp(s^-_k - s^+_k))$ \tcp*{Margin-based contrastive loss}

$\mathcal{L}_d \gets - \mathcal{S}(q, d_{1:N_d}^+) + \mathcal{S}(q, d_1^+) + \mathcal{S}(q, d_{1:N_d}^-) - \mathcal{S}(q, d_1^-)$ \tcp*{Progressive refinement}

$\mathcal{L}_q \gets \frac{1}{Z} \sum_{i=1}^{N_q} \sum_{j\neq i} \left( \frac{q_i^\top q_j}{\|q_i\|\|q_j\|} \right)^2$ \tcp*{Diversity regularization}

\Return $\mathcal{L} = \lambda_m \mathcal{L}_m + \lambda_d \mathcal{L}_d + \lambda_q \mathcal{L}_q$
\end{algorithm}

%% file: reference.bib
@misc{radford2021clip,
      title={Learning Transferable Visual Models From Natural Language Supervision}, 
      author={Alec Radford and Jong Wook Kim and Chris Hallacy and Aditya Ramesh and Gabriel Goh and Sandhini Agarwal and Girish Sastry and Amanda Askell and Pamela Mishkin and Jack Clark and Gretchen Krueger and Ilya Sutskever},
      year={2021},
      eprint={2103.00020},
      archivePrefix={arXiv},
      primaryClass={cs.CV},
      url={https://arxiv.org/abs/2103.00020}, 
}

@inproceedings{li2022blip,
  title={Blip: Bootstrapping language-image pre-training for unified vision-language understanding and generation},
  author={Li, Junnan and Li, Dongxu and Xiong, Caiming and Hoi, Steven},
  booktitle={International conference on machine learning},
  pages={12888--12900},
  year={2022},
  organization={PMLR}
}

@inproceedings{zhai2023siglip,
  title={Sigmoid loss for language image pre-training},
  author={Zhai, Xiaohua and Mustafa, Basil and Kolesnikov, Alexander and Beyer, Lucas},
  booktitle={Proceedings of the IEEE/CVF international conference on computer vision},
  pages={11975--11986},
  year={2023}
}

@article{jiang2024vlm2vec,
  title={Vlm2vec: Training vision-language models for massive multimodal embedding tasks},
  author={Jiang, Ziyan and Meng, Rui and Yang, Xinyi and Yavuz, Semih and Zhou, Yingbo and Chen, Wenhu},
  journal={arXiv preprint arXiv:2410.05160},
  year={2024}
}

@article{meng2025vlm2vec2,
  title={Vlm2vec-v2: Advancing multimodal embedding for videos, images, and visual documents},
  author={Meng, Rui and Jiang, Ziyan and Liu, Ye and Su, Mingyi and Yang, Xinyi and Fu, Yuepeng and Qin, Can and Chen, Zeyuan and Xu, Ran and Xiong, Caiming and others},
  journal={arXiv preprint arXiv:2507.04590},
  year={2025}
}

@inproceedings{gunther2025jinav4,
  title={jina-embeddings-v4: Universal embeddings for multimodal multilingual retrieval},
  author={G{\"u}nther, Michael and Sturua, Saba and Akram, Mohammad Kalim and Mohr, Isabelle and Ungureanu, Andrei and Wang, Bo and Eslami, Sedigheh and Martens, Scott and Werk, Maximilian and Wang, Nan and others},
  booktitle={Proceedings of the 5th Workshop on Multilingual Representation Learning (MRL 2025)},
  pages={531--550},
  year={2025}
}

@misc{nomicembedmultimodal2025,
  title={Nomic Embed Multimodal: Interleaved Text, Image, and Screenshots for Visual Document Retrieval},
  author={Nomic Team},
  year={2025},
  publisher={Nomic AI},
  url={https://nomic.ai/blog/posts/nomic-embed-multimodal},
}

@article{li2025coma,
  title={Compression then Matching: An Efficient Pre-training Paradigm for Multimodal Embedding},
  author={Li, Da and Luo, Yuxiao and Bi, Keping and Guo, Jiafeng and Yuan, Wei and Yang, Biao and Wang, Yan and Yang, Fan and Gao, Tingting and Zhou, Guorui},
  journal={arXiv preprint arXiv:2511.08480},
  year={2025}
}

@article{jian2025rzenembed,
  title={RzenEmbed: Towards Comprehensive Multimodal Retrieval},
  author={Jian, Weijian and Zhang, Yajun and Liang, Dawei and Xie, Chunyu and He, Yixiao and Leng, Dawei and Yin, Yuhui},
  journal={arXiv preprint arXiv:2510.27350},
  year={2025}
}

@inproceedings{mcqueen1967kmeans,
  title={Some methods of classification and analysis of multivariate observations},
  author={McQueen, James B},
  booktitle={Proc. of 5th Berkeley Symposium on Math. Stat. and Prob.},
  pages={281--297},
  year={1967}
}

@article{cui2025think,
  title={Think then embed: Generative context improves multimodal embedding},
  author={Cui, Xuanming and Cheng, Jianpeng and Chen, Hong-you and Shukla, Satya Narayan and Awasthi, Abhijeet and Pan, Xichen and Ahuja, Chaitanya and Mishra, Shlok Kumar and Yang, Yonghuan and Xiao, Jun and others},
  journal={arXiv preprint arXiv:2510.05014},
  year={2025}
}

@article{tsai2025giecse,
  title={Let LLMs Speak Embedding Languages: Generative Text Embeddings via Iterative Contrastive Refinement},
  author={Tsai, Yu-Che and Chen, Kuan-Yu and Li, Yuan-Chi and Chen, Yuan-Hao and Tsai, Ching-Yu and Lin, Shou-De},
  journal={arXiv preprint arXiv:2509.24291},
  year={2025}
}

@article{lan2025llave,
  title={Llave: Large language and vision embedding models with hardness-weighted contrastive learning},
  author={Lan, Zhibin and Niu, Liqiang and Meng, Fandong and Zhou, Jie and Su, Jinsong},
  journal={arXiv preprint arXiv:2503.04812},
  year={2025}
}

@misc{abdin2024phi3,
      title={Phi-3 Technical Report: A Highly Capable Language Model Locally on Your Phone}, 
      author={Marah Abdin and Jyoti Aneja and Hany Awadalla and Ahmed Awadallah and Ammar Ahmad Awan and Nguyen Bach and Amit Bahree and Arash Bakhtiari and Jianmin Bao and Harkirat Behl and Alon Benhaim and Misha Bilenko and Johan Bjorck and Sébastien Bubeck and Martin Cai and Qin Cai and Vishrav Chaudhary and Dong Chen and Dongdong Chen and Weizhu Chen and Yen-Chun Chen and Yi-Ling Chen and Hao Cheng and Parul Chopra and Xiyang Dai and Matthew Dixon and Ronen Eldan and Victor Fragoso and Jianfeng Gao and Mei Gao and Min Gao and Amit Garg and Allie Del Giorno and Abhishek Goswami and Suriya Gunasekar and Emman Haider and Junheng Hao and Russell J. Hewett and Wenxiang Hu and Jamie Huynh and Dan Iter and Sam Ade Jacobs and Mojan Javaheripi and Xin Jin and Nikos Karampatziakis and Piero Kauffmann and Mahoud Khademi and Dongwoo Kim and Young Jin Kim and Lev Kurilenko and James R. Lee and Yin Tat Lee and Yuanzhi Li and Yunsheng Li and Chen Liang and Lars Liden and Xihui Lin and Zeqi Lin and Ce Liu and Liyuan Liu and Mengchen Liu and Weishung Liu and Xiaodong Liu and Chong Luo and Piyush Madan and Ali Mahmoudzadeh and David Majercak and Matt Mazzola and Caio César Teodoro Mendes and Arindam Mitra and Hardik Modi and Anh Nguyen and Brandon Norick and Barun Patra and Daniel Perez-Becker and Thomas Portet and Reid Pryzant and Heyang Qin and Marko Radmilac and Liliang Ren and Gustavo de Rosa and Corby Rosset and Sambudha Roy and Olatunji Ruwase and Olli Saarikivi and Amin Saied and Adil Salim and Michael Santacroce and Shital Shah and Ning Shang and Hiteshi Sharma and Yelong Shen and Swadheen Shukla and Xia Song and Masahiro Tanaka and Andrea Tupini and Praneetha Vaddamanu and Chunyu Wang and Guanhua Wang and Lijuan Wang and Shuohang Wang and Xin Wang and Yu Wang and Rachel Ward and Wen Wen and Philipp Witte and Haiping Wu and Xiaoxia Wu and Michael Wyatt and Bin Xiao and Can Xu and Jiahang Xu and Weijian Xu and Jilong Xue and Sonali Yadav and Fan Yang and Jianwei Yang and Yifan Yang and Ziyi Yang and Donghan Yu and Lu Yuan and Chenruidong Zhang and Cyril Zhang and Jianwen Zhang and Li Lyna Zhang and Yi Zhang and Yue Zhang and Yunan Zhang and Xiren Zhou},
      year={2024},
      eprint={2404.14219},
      archivePrefix={arXiv},
      primaryClass={cs.CL},
      url={https://arxiv.org/abs/2404.14219}, 
}

@article{li2024llavaov,
  title={Llava-onevision: Easy visual task transfer},
  author={Li, Bo and Zhang, Yuanhan and Guo, Dong and Zhang, Renrui and Li, Feng and Zhang, Hao and Zhang, Kaichen and Zhang, Peiyuan and Li, Yanwei and Liu, Ziwei and others},
  journal={arXiv preprint arXiv:2408.03326},
  year={2024}
}

@misc{wang2024qwen2vl,
      title={Qwen2-VL: Enhancing Vision-Language Model's Perception of the World at Any Resolution}, 
      author={Peng Wang and Shuai Bai and Sinan Tan and Shijie Wang and Zhihao Fan and Jinze Bai and Keqin Chen and Xuejing Liu and Jialin Wang and Wenbin Ge and Yang Fan and Kai Dang and Mengfei Du and Xuancheng Ren and Rui Men and Dayiheng Liu and Chang Zhou and Jingren Zhou and Junyang Lin},
      year={2024},
      eprint={2409.12191},
      archivePrefix={arXiv},
      primaryClass={cs.CV},
      url={https://arxiv.org/abs/2409.12191}, 
}

@misc{bai2025qwen25vl,
      title={Qwen2.5-VL Technical Report}, 
      author={Shuai Bai and Keqin Chen and Xuejing Liu and Jialin Wang and Wenbin Ge and Sibo Song and Kai Dang and Peng Wang and Shijie Wang and Jun Tang and Humen Zhong and Yuanzhi Zhu and Mingkun Yang and Zhaohai Li and Jianqiang Wan and Pengfei Wang and Wei Ding and Zheren Fu and Yiheng Xu and Jiabo Ye and Xi Zhang and Tianbao Xie and Zesen Cheng and Hang Zhang and Zhibo Yang and Haiyang Xu and Junyang Lin},
      year={2025},
      eprint={2502.13923},
      archivePrefix={arXiv},
      primaryClass={cs.CV},
      url={https://arxiv.org/abs/2502.13923}, 
}

@misc{bai2025qwen3vl,
      title={Qwen3-VL Technical Report}, 
      author={Shuai Bai and Yuxuan Cai and Ruizhe Chen and Keqin Chen and Xionghui Chen and Zesen Cheng and Lianghao Deng and Wei Ding and Chang Gao and Chunjiang Ge and Wenbin Ge and Zhifang Guo and Qidong Huang and Jie Huang and Fei Huang and Binyuan Hui and Shutong Jiang and Zhaohai Li and Mingsheng Li and Mei Li and Kaixin Li and Zicheng Lin and Junyang Lin and Xuejing Liu and Jiawei Liu and Chenglong Liu and Yang Liu and Dayiheng Liu and Shixuan Liu and Dunjie Lu and Ruilin Luo and Chenxu Lv and Rui Men and Lingchen Meng and Xuancheng Ren and Xingzhang Ren and Sibo Song and Yuchong Sun and Jun Tang and Jianhong Tu and Jianqiang Wan and Peng Wang and Pengfei Wang and Qiuyue Wang and Yuxuan Wang and Tianbao Xie and Yiheng Xu and Haiyang Xu and Jin Xu and Zhibo Yang and Mingkun Yang and Jianxin Yang and An Yang and Bowen Yu and Fei Zhang and Hang Zhang and Xi Zhang and Bo Zheng and Humen Zhong and Jingren Zhou and Fan Zhou and Jing Zhou and Yuanzhi Zhu and Ke Zhu},
      year={2025},
      eprint={2511.21631},
      archivePrefix={arXiv},
      primaryClass={cs.CV},
      url={https://arxiv.org/abs/2511.21631}, 
}

@article{zhang2025mmretrieval_survey,
  title={Composed multi-modal retrieval: A survey of approaches and applications},
  author={Zhang, Kun and Li, Jingyu and Li, Zhe and Zhang, Jingjing and Li, Fan and Liu, Yandong and Yan, Rui and Jiang, Zihang and Chen, Nan and Zhang, Lei and others},
  journal={arXiv preprint arXiv:2503.01334},
  year={2025}
}

@article{gao2025scaling,
  title={Scaling Beyond Context: A Survey of Multimodal Retrieval-Augmented Generation for Document Understanding},
  author={Gao, Sensen and Zhao, Shanshan and Jiang, Xu and Duan, Lunhao and Chng, Yong Xien and Chen, Qing-Guo and Luo, Weihua and Zhang, Kaifu and Bian, Jia-Wang and Gong, Mingming},
  journal={arXiv preprint arXiv:2510.15253},
  year={2025}
}

@article{zheng2025retrieval,
  title={Retrieval augmented generation and understanding in vision: A survey and new outlook},
  author={Zheng, Xu and Weng, Ziqiao and Lyu, Yuanhuiyi and Jiang, Lutao and Xue, Haiwei and Ren, Bin and Paudel, Danda and Sebe, Nicu and Van Gool, Luc and Hu, Xuming},
  journal={arXiv preprint arXiv:2503.18016},
  year={2025}
}

@article{yan2025docpruner,
  title={Docpruner: A storage-efficient framework for multi-vector visual document retrieval via adaptive patch-level embedding pruning},
  author={Yan, Yibo and Xu, Guangwei and Zou, Xin and Liu, Shuliang and Kwok, James and Hu, Xuming},
  journal={arXiv preprint arXiv:2509.23883},
  year={2025}
}

@inproceedings{zhang2025roles,
  title={Roles of MLLMs in Visually Rich Document Retrieval for RAG: A Survey},
  author={Zhang, Xiantao},
  booktitle={Proceedings of the 14th International Joint Conference on Natural Language Processing and the 4th Conference of the Asia-Pacific Chapter of the Association for Computational Linguistics},
  pages={19--36},
  year={2025}
}

@article{xiao2025metaembed,
  title={Metaembed: Scaling multimodal retrieval at test-time with flexible late interaction},
  author={Xiao, Zilin and Ma, Qi and Gu, Mengting and Chen, Chun-cheng Jason and Chen, Xintao and Ordonez, Vicente and Mohan, Vijai},
  journal={arXiv preprint arXiv:2509.18095},
  year={2025}
}

@inproceedings{most2025lost,
  title={Lost in ocr translation? vision-based approaches to robust document retrieval},
  author={Most, Alexander and Winjum, Joseph and Bhattarai, Manish and Jones, Shawn and Ranasinghe, Nishath Rajiv and Biswas, Ayan and O'Malley, Dan},
  booktitle={Proceedings of the 2025 ACM Symposium on Document Engineering},
  pages={1--10},
  year={2025}
}

@inproceedings{zhang2025ocr,
  title={Ocr hinders rag: Evaluating the cascading impact of ocr on retrieval-augmented generation},
  author={Zhang, Junyuan and Zhang, Qintong and Wang, Bin and Ouyang, Linke and Wen, Zichen and Li, Ying and Chow, Ka-Ho and He, Conghui and Zhang, Wentao},
  booktitle={Proceedings of the IEEE/CVF International Conference on Computer Vision},
  pages={17443--17453},
  year={2025}
}

@article{faysse2024colpali,
  title={Colpali: Efficient document retrieval with vision language models},
  author={Faysse, Manuel and Sibille, Hugues and Wu, Tony and Omrani, Bilel and Viaud, Gautier and Hudelot, C{\'e}line and Colombo, Pierre},
  journal={arXiv preprint arXiv:2407.01449},
  year={2024}
}

@article{ma2024unifying,
  title={Unifying multimodal retrieval via document screenshot embedding},
  author={Ma, Xueguang and Lin, Sheng-Chieh and Li, Minghan and Chen, Wenhu and Lin, Jimmy},
  journal={arXiv preprint arXiv:2406.11251},
  year={2024}
}

@article{zhang2024gme,
  title={GME: Improving Universal Multimodal Retrieval by Multimodal LLMs},
  author={Zhang, Xin and Zhang, Yanzhao and Xie, Wen and Li, Mingxin and Dai, Ziqi and Long, Dingkun and Xie, Pengjun and Zhang, Meishan and Li, Wenjie and Zhang, Min},
  journal={arXiv preprint arXiv:2412.16855},
  year={2024}
}

@inproceedings{liu2025any,
  title={Any information is just worth one single screenshot: Unifying search with visualized information retrieval},
  author={Liu, Zheng and Liu, Ze and Liang, Zhengyang and Zhou, Junjie and Xiao, Shitao and Gao, Chao and Zhang, Chen Jason and Lian, Defu},
  booktitle={Proceedings of the 63rd Annual Meeting of the Association for Computational Linguistics (Volume 1: Long Papers)},
  pages={19238--19261},
  year={2025}
}

@article{cha2026reinpool,
  title={ReinPool: Reinforcement Learning Pooling Multi-Vector Embeddings for Retrieval System},
  author={Cha, Sungguk and Kim, DongWook and Kim, Mintae and Han, Youngsub and Jeon, Byoung-Ki and Lee, Sangyeob},
  journal={arXiv preprint arXiv:2601.07125},
  year={2026}
}

@article{mace2025vidore2,
  title={ViDoRe Benchmark V2: Raising the Bar for Visual Retrieval},
  author={Mac{\'e}, Quentin and Loison, Ant{\'o}nio and Faysse, Manuel},
  journal={arXiv preprint arXiv:2505.17166},
  year={2025}
}

@article{loison2026vidore3,
  title={ViDoRe V3: A Comprehensive Evaluation of Retrieval Augmented Generation in Complex Real-World Scenarios},
  author={Loison, Ant{\'o}nio and Mac{\'e}, Quentin and Edy, Antoine and Xing, Victor and Balough, Tom and Moreira, Gabriel and Liu, Bo and Faysse, Manuel and Hudelot, C{\'e}line and Viaud, Gautier},
  journal={arXiv preprint arXiv:2601.08620},
  year={2026}
}

@article{EagerEmbed,
  title={Eager Embed V1: Multimodal Dense Embeddings for Retrieval},
  author={Juan Pablo Balarini},
  year={2025},
  publisher={Eagerworks},
  journal = {GitHub repository},
  url={https://github.com/eagerworks/eager-embed}
}

@inproceedings{smith2007tesseract,
  title={An overview of the Tesseract OCR engine},
  author={Smith, Ray},
  booktitle={Ninth international conference on document analysis and recognition (ICDAR 2007)},
  volume={2},
  pages={629--633},
  year={2007},
  organization={IEEE}
}

@inproceedings{khattab2020colbert,
  title={Colbert: Efficient and effective passage search via contextualized late interaction over bert},
  author={Khattab, Omar and Zaharia, Matei},
  booktitle={Proceedings of the 43rd International ACM SIGIR conference on research and development in Information Retrieval},
  pages={39--48},
  year={2020}
}

@inproceedings{santhanam2022colbertv2,
  title={Colbertv2: Effective and efficient retrieval via lightweight late interaction},
  author={Santhanam, Keshav and Khattab, Omar and Saad-Falcon, Jon and Potts, Christopher and Zaharia, Matei},
  booktitle={Proceedings of the 2022 Conference of the North American Chapter of the Association for Computational Linguistics: Human Language Technologies},
  pages={3715--3734},
  year={2022}
}

@misc{ops_colqwen3_4b,
  author       = {{OpenSearch-AI}},
  title        = {{Ops-ColQwen3: State-of-the-Art Multimodal Embedding Model for Visual Document Retrieval}},
  year         = {2026},
  howpublished = {\url{https://huggingface.co/OpenSearch-AI/Ops-ColQwen3-4B}},
}

@misc{barboule2025surveyvdr,
      title={Survey on Question Answering over Visually Rich Documents: Methods, Challenges, and Trends}, 
      author={Camille Barboule and Benjamin Piwowarski and Yoan Chabot},
      year={2025},
      eprint={2501.02235},
      archivePrefix={arXiv},
      primaryClass={cs.CL},
      url={https://arxiv.org/abs/2501.02235}, 
}

@misc{huang2025tomoro,
  author = {Huang, Xin and Tan, Kye Min},
  title = {Beyond Text: Unlocking True Multimodal, End-to-end RAG with Tomoro ColQwen3},
  year = {2025},
  url = {https://tomoro.ai/insights/beyond-text-unlocking-true-multimodal-end-to-end-rag-with-tomoro-colqwen3},
  publisher = {Tomoro.ai}
}

@article{qwen3vlembedding,
  title={Qwen3-VL-Embedding and Qwen3-VL-Reranker: A Unified Framework for State-of-the-Art Multimodal Retrieval and Ranking},
  author={Li, Mingxin and Zhang, Yanzhao and Long, Dingkun and Chen Keqin and Song, Sibo and Bai, Shuai and Yang, Zhibo and Xie, Pengjun and Yang, An and Liu, Dayiheng and Zhou, Jingren and Lin, Junyang},
  journal={arXiv preprint arXiv:2601.04720},
  year={2026}
}

@article{zhuang2024starbucksv2,
  title={Starbucks-v2: Improved Training for 2D Matryoshka Embeddings},
  author={Zhuang, Shengyao and Wang, Shuai and Zheng, Fabio and Koopman, Bevan and Zuccon, Guido},
  journal={arXiv preprint arXiv:2410.13230},
  year={2024}
}

@misc{EvoQwen,
    author       = { ApsaraStackMaaS },
    title        = {EvoQwen2.5-VL-Retriever-3B-v1 },
    year         = 2026,
    url          = { https://huggingface.co/ApsaraStackMaaS/EvoQwen2.5-VL-Retriever-3B-v1},
    publisher    = { Hugging Face }
}

@article{xu2025llama,
  title={Llama nemoretriever colembed: Top-performing text-image retrieval model},
  author={Xu, Mengyao and Moreira, Gabriel and Ak, Ronay and Osmulski, Radek and Babakhin, Yauhen and Yu, Zhiding and Schifferer, Benedikt and Oldridge, Even},
  journal={arXiv preprint arXiv:2507.05513},
  year={2025}
}

@article{grattafiori2024llama3,
  title={The llama 3 herd of models},
  author={Grattafiori, Aaron and Dubey, Abhimanyu and Jauhri, Abhinav and Pandey, Abhinav and Kadian, Abhishek and Al-Dahle, Ahmad and Letman, Aiesha and Mathur, Akhil and Schelten, Alan and Vaughan, Alex and others},
  journal={arXiv preprint arXiv:2407.21783},
  year={2024}
}

@article{liu2025reasoning,
  title={Reasoning Guided Embeddings: Leveraging MLLM Reasoning for Improved Multimodal Retrieval},
  author={Liu, Chunxu and Yang, Jiyuan and Gao, Ruopeng and Zhu, Yuhan and Zhu, Feng and Zhao, Rui and Wang, Limin},
  journal={arXiv preprint arXiv:2511.16150},
  year={2025}
}

@article{lan2025ume,
  title={UME-R1: Exploring Reasoning-Driven Generative Multimodal Embeddings},
  author={Lan, Zhibin and Niu, Liqiang and Meng, Fandong and Zhou, Jie and Su, Jinsong},
  journal={arXiv preprint arXiv:2511.00405},
  year={2025}
}

@article{beyer2024paligemma,
    title={{PaliGemma: A versatile 3B VLM for transfer}},
    author={Lucas Beyer* and Andreas Steiner* and André Susano Pinto* and Alexander Kolesnikov* and Xiao Wang* and Daniel Salz and Maxim Neumann and Ibrahim Alabdulmohsin and Michael Tschannen and Emanuele Bugliarello and Thomas Unterthiner and Daniel Keysers and Skanda Koppula and Fangyu Liu and Adam Grycner and Alexey Gritsenko and Neil Houlsby and Manoj Kumar and Keran Rong and Julian Eisenschlos and Rishabh Kabra and Matthias Bauer and Matko Bo{\v{s}}njak and Xi Chen and Matthias Minderer and Paul Voigtlaender and Ioana Bica and Ivana Balazevic and Joan Puigcerver and Pinelopi Papalampidi and Olivier Henaff and Xi Xiong and Radu Soricut and Jeremiah Harmsen and Xiaohua Zhai*},
    year={2024},
    journal={arXiv preprint arXiv:2407.07726}
}

@article{kusupati2022matryoshka,
  title={Matryoshka representation learning},
  author={Kusupati, Aditya and Bhatt, Gantavya and Rege, Aniket and Wallingford, Matthew and Sinha, Aditya and Ramanujan, Vivek and Howard-Snyder, William and Chen, Kaifeng and Kakade, Sham and Jain, Prateek and others},
  journal={Advances in Neural Information Processing Systems},
  volume={35},
  pages={30233--30249},
  year={2022}
}

@article{yu2025cafe,
  title={CAFe: Unifying Representation and Generation with Contrastive-Autoregressive Finetuning},
  author={Yu, Hao and Zhao, Zhuokai and Yan, Shen and Korycki, Lukasz and Wang, Jianyu and He, Baosheng and Liu, Jiayi and Zhang, Lizhu and Fan, Xiangjun and Yu, Hanchao},
  journal={arXiv preprint arXiv:2503.19900},
  year={2025}
}

@article{hao2024coconut,
  title={Training Large Language Models to Reason in a Continuous Latent Space},
  author={Hao, Shibo and Sukhbaatar, Sainbayar and Su, DiJia and Li, Xian and Hu, Zhiting and Weston, Jason and Tian, Yuandong},
  journal={arXiv preprint arXiv:2412.06769},
  year={2024}
}

@article{team2025gemma3,
  title={Gemma 3 technical report},
  author={Team, Gemma and Kamath, Aishwarya and Ferret, Johan and Pathak, Shreya and Vieillard, Nino and Merhej, Ramona and Perrin, Sarah and Matejovicova, Tatiana and Ram{\'e}, Alexandre and Rivi{\`e}re, Morgane and others},
  journal={arXiv preprint arXiv:2503.19786},
  year={2025}
}

@article{tian2024var,
  title={Visual autoregressive modeling: Scalable image generation via next-scale prediction},
  author={Tian, Keyu and Jiang, Yi and Yuan, Zehuan and Peng, Bingyue and Wang, Liwei},
  journal={Advances in neural information processing systems},
  volume={37},
  pages={84839--84865},
  year={2024}
}

@article{brown2020gpt,
  title={Language models are few-shot learners},
  author={Brown, Tom and Mann, Benjamin and Ryder, Nick and Subbiah, Melanie and Kaplan, Jared D and Dhariwal, Prafulla and Neelakantan, Arvind and Shyam, Pranav and Sastry, Girish and Askell, Amanda and others},
  journal={Advances in neural information processing systems},
  volume={33},
  pages={1877--1901},
  year={2020}
}

@article{ma2025lightcolpali,
  title={Towards Storage-Efficient Visual Document Retrieval: An Empirical Study on Reducing Patch-Level Embeddings},
  author={Ma, Yubo and Li, Jinsong and Zang, Yuhang and Wu, Xiaobao and Dong, Xiaoyi and Zhang, Pan and Cao, Yuhang and Duan, Haodong and Wang, Jiaqi and Cao, Yixin and others},
  journal={arXiv preprint arXiv:2506.04997},
  year={2025}
}

@misc{kolavi2025colnetraembed,
  title={M3DR: Towards Universal Multilingual Multimodal Document Retrieval}, 
  author={Adithya S Kolavi and Vyoman Jain},
  year={2025},
  eprint={2512.03514},
  archivePrefix={arXiv},
  primaryClass={cs.IR},
  url={https://arxiv.org/abs/2512.03514}
}

@article{achiam2023gpt,
  title={Gpt-4 technical report},
  author={Achiam, Josh and Adler, Steven and Agarwal, Sandhini and Ahmad, Lama and Akkaya, Ilge and Aleman, Florencia Leoni and Almeida, Diogo and Altenschmidt, Janko and Altman, Sam and Anadkat, Shyamal and others},
  journal={arXiv preprint arXiv:2303.08774},
  year={2023}
}

@article{team2023gemini,
  title={Gemini: a family of highly capable multimodal models},
  author={Team, Gemini and Anil, Rohan and Borgeaud, Sebastian and Alayrac, Jean-Baptiste and Yu, Jiahui and Soricut, Radu and Schalkwyk, Johan and Dai, Andrew M and Hauth, Anja and Millican, Katie and others},
  journal={arXiv preprint arXiv:2312.11805},
  year={2023}
}

@article{guo2025deepseek,
  title={DeepSeek-R1 incentivizes reasoning in LLMs through reinforcement learning},
  author={Guo, Daya and Yang, Dejian and Zhang, Haowei and Song, Junxiao and Wang, Peiyi and Zhu, Qihao and Xu, Runxin and Zhang, Ruoyu and Ma, Shirong and Bi, Xiao and others},
  journal={Nature},
  volume={645},
  number={8081},
  pages={633--638},
  year={2025},
  publisher={Nature Publishing Group UK London}
}

@article{sun2024autoregressive,
  title={Autoregressive model beats diffusion: Llama for scalable image generation},
  author={Sun, Peize and Jiang, Yi and Chen, Shoufa and Zhang, Shilong and Peng, Bingyue and Luo, Ping and Yuan, Zehuan},
  journal={arXiv preprint arXiv:2406.06525},
  year={2024}
}

@inproceedings{wu2025janus,
  title={Janus: Decoupling visual encoding for unified multimodal understanding and generation},
  author={Wu, Chengyue and Chen, Xiaokang and Wu, Zhiyu and Ma, Yiyang and Liu, Xingchao and Pan, Zizheng and Liu, Wen and Xie, Zhenda and Yu, Xingkai and Ruan, Chong and others},
  booktitle={Proceedings of the Computer Vision and Pattern Recognition Conference},
  pages={12966--12977},
  year={2025}
}

@article{chen2025janus,
  title={Janus-pro: Unified multimodal understanding and generation with data and model scaling},
  author={Chen, Xiaokang and Wu, Zhiyu and Liu, Xingchao and Pan, Zizheng and Liu, Wen and Xie, Zhenda and Yu, Xingkai and Ruan, Chong},
  journal={arXiv preprint arXiv:2501.17811},
  year={2025}
}

@inproceedings{liu2025javisgpt,
    title={JavisGPT: A Unified Multi-modal LLM for Sounding-Video Comprehension and Generation},
    author={Kai Liu and Jungang Li and Yuchong Sun and Shengqiong Wu and jianzhang gao and Daoan Zhang and Wei Zhang and Sheng Jin and Sicheng Yu and Geng Zhan and Jiayi Ji and Fan Zhou and Liang Zheng and Shuicheng YAN and Hao Fei and Tat-Seng Chua},
    booktitle={The Thirty-ninth Annual Conference on Neural Information Processing Systems},
    year={2025},
}

@article{lin2025causal2vec,
  title={Causal2Vec: Improving Decoder-only LLMs as Versatile Embedding Models},
  author={Lin, Ailiang and Li, Zhuoyun and Funakoshi, Kotaro and Okumura, Manabu},
  journal={arXiv preprint arXiv:2507.23386},
  year={2025}
}

@article{jayaram2024muvera,
  title={MUVERA: Multi-Vector Retrieval via Fixed Dimensional Encoding},
  author={Jayaram, Rajesh and Dhulipala, Laxman and Hadian, Majid and Lee, Jason D and Mirrokni, Vahab},
  journal={Advances in Neural Information Processing Systems},
  volume={37},
  pages={101042--101073},
  year={2024}
}

@article{zhou2024vista,
  title={Vista: Visualized text embedding for universal multi-modal retrieval},
  author={Zhou, Junjie and Liu, Zheng and Xiao, Shitao and Zhao, Bo and Xiong, Yongping},
  journal={arXiv preprint arXiv:2406.04292},
  year={2024}
}
